\documentclass[oneside,11pt]{article} 
\renewenvironment{abstract}
  {{\centering\large\bfseries Abstract\par}\vspace{0.7ex}%
    \bgroup
       \leftskip 20pt\rightskip 20pt\small\noindent\ignorespaces}%
  {\par\egroup\vskip 0.25ex}

\newenvironment{keywords}
{\vspace{0.05in}\bgroup\leftskip 20pt\rightskip 20pt \small\noindent{\bfseries
Keywords:} \ignorespaces}
{\par\egroup\vskip 0.25ex}

\usepackage[utf8]{inputenc} 
\usepackage[T1]{fontenc}    
\usepackage{hyperref}       
\usepackage{url}            
\usepackage{booktabs}       
\usepackage{nicefrac}       
\usepackage{microtype}      

\usepackage{graphicx}
\usepackage{algorithm,algorithmic}
\usepackage{amsfonts,amssymb,amsmath,commath,amsthm}
\usepackage{xcolor}
\usepackage{booktabs}
\usepackage{enumitem}

\usepackage{pifont}

\numberwithin{equation}{section}

\theoremstyle{plain}
\newtheorem{theorem}{Theorem}
\newtheorem{lemma}[theorem]{Lemma}
\newtheorem{proposition}[theorem]{Proposition}
\newtheorem{corollary}[theorem]{Corollary}
\theoremstyle{definition}
\newtheorem{definition}[theorem]{Definition}
\newtheorem{remark}{Remark}
\newtheorem{claim}{Claim}[section]
\newtheorem{assumption}{Assumption}

\newcommand{\defeq}{:=}

\renewcommand{\(}{\left(}
\renewcommand{\)}{\right)}

\DeclareMathOperator*{\ve}{vec}
\newcommand{\st}{\mathrm{s.t.}}
\newcommand{\poly}[1]{\operatorname{poly}\del{#1}}
\newcommand{\polylog}[1]{\operatorname{polylog}\del{#1}}
\newcommand{\sign}[1]{\operatorname{sign}\del{#1}}
\renewcommand{\Pr}{\operatorname{Pr}}
\newcommand{\EXP}{\operatorname{\mathbb{E}}}
\newcommand{\EE}{\mathbb{E}}
\newcommand{\RR}{\mathbb{R}}

\newcommand{\err}{\operatorname{err}}

\newcommand{\trans}{^{\top}}
\newcommand{\inner}[2]{\left\langle #1, #2 \right\rangle}

\newcommand{\ind}[1]{\boldsymbol{1}_{\cbr{#1}}}

\newcommand{\R}{\mathbb{R}}
\newcommand{\Rd}{\mathbb{R}^d}
\newcommand{\Rdd}{\mathbb{R}^{d\times d}}

\newcommand{\calC}{\mathcal{C}}
\newcommand{\calD}{\mathcal{D}}

\newcommand{\calH}{\mathcal{H}}
\newcommand{\calF}{\mathcal{F}}
\newcommand{\calG}{\mathcal{G}}
\newcommand{\calR}{\mathcal{R}}

\newcommand{\calM}{\mathcal{M}}

\newcommand{\twonorm}[1]{\left\lVert #1 \right\rVert_{2}}
\newcommand{\orcnorm}[2]{\left\lVert #1 \right\rVert_{\psi^{}_{#2}}}
\newcommand{\onenorm}[1]{\left\lVert #1 \right\rVert_{1}}
\renewcommand{\abs}[1]{\left\lvert #1 \right\rvert}

\newcommand{\nuclearnorm}[1]{\left\lVert #1 \right\rVert_{*}}

\newcommand{\spenorm}[1]{\left\lVert #1 \right\rVert}
\newcommand{\infnorm}[1]{\left\lVert #1 \right\rVert_{\infty}}
\renewcommand{\norm}[1]{\left\lVert #1 \right\rVert}

\newcommand{\oraclexy}{\mathrm{EX}_{\eta}(D, w^*)}
\newcommand{\oraclex}{\mathrm{EX}_{\eta}^x(D, w^*)}
\newcommand{\oracley}{\mathrm{EX}_{\eta}^y(D, w^*)}

\newcommand{\TC}{T_{\mathrm{C}}}
\newcommand{\TD}{T_{\mathrm{D}}}

\newcommand{\tildeTC}{\tilde{T}_{\mathrm{C}}}
\newcommand{\barTC}{\bar{T}_{\mathrm{C}}}
\newcommand{\hatTC}{\hat{T}_{\mathrm{C}}}
\newcommand{\barTD}{\bar{T}_{\mathrm{D}}}

\newcommand{\pD}{p^{}_{\mathrm{D}}}

\newcommand{\citep}[1]{\cite{#1}}
\newcommand{\citet}[1]{\cite{#1}}

\title{Attribute-Efficient Learning of Halfspaces with Malicious Noise: Near-Optimal Label Complexity and Noise Tolerance}

\author{%
  Jie Shen \\
Stevens Institute of Technology\\
  \texttt{jie.shen@stevens.edu} \\
\and
Chicheng Zhang\\
University of Arizona\\
\texttt{chichengz@cs.arizona.edu}
}

\begin{document}

\maketitle

\begin{abstract}
This paper is concerned with computationally efficient learning of homogeneous sparse halfspaces in $\Rd$ under noise. Though recent works have established attribute-efficient learning algorithms under various types of label noise (e.g. bounded noise), it remains an open question of when and how $s$-sparse halfspaces can be efficiently learned under the challenging {\em malicious noise} model, where an adversary may corrupt both the unlabeled examples and the labels. We answer this question in the affirmative by designing a computationally efficient active learning algorithm with near-optimal label complexity of $\tilde{O}(s \log^4\frac{d}{\epsilon})$\footnote{We use the notation $\tilde{O}(f) \defeq O(f \log f)$.}
and noise tolerance $\eta = \Omega(\epsilon)$, where $\epsilon \in (0, 1)$ is the target error rate, under the assumption that the distribution over (uncorrupted) unlabeled examples is isotropic log-concave. Our algorithm can be straightforwardly tailored to the passive learning setting, and we show that its sample complexity is $\tilde{O}(\frac{1}{\epsilon}s^2 \log^5 d)$ which also enjoys attribute efficiency. Our main techniques include attribute-efficient paradigms for soft outlier removal and for empirical risk minimization, and a new analysis of uniform concentration for unbounded instances~--~all of them crucially take the sparsity structure of the underlying halfspace into account.
\end{abstract}



%

%
%
%
%

\section{Introduction}\label{sec:intro}

This paper investigates the fundamental problem of learning halfspaces under noise~\citep{valiant1984theory,valiant1985learning}. In the absence of noise, this problem is well understood~\citep{rosenblatt1958perceptron,blumer1989learn}. However, the premise changes immediately when the unlabeled examples\footnote{We will also refer to unlabeled examples as instances in this paper.} or the labels are corrupted by noise. In the last decades, various types of label noise have been extensively studied, and a plethora of polynomial-time algorithms have been developed that are resilient to random classification noise~\citep{blum1996polynomial}, bounded noise~\citep{sloan1988types,sloan1992Corrigendum,massart2006risk}, and adversarial noise~\citep{kearns1992toward,kalai2005agnostic}. Significant progress towards optimal noise tolerance is also witnessed in the past few years~\citep{daniely2015ptas,awasthi2015efficient,yan2017revisiting,diakonikolas2019distribution,diakonikolas2020learning}. In this regard, a surge of recent research interest is concentrated on further improvement of the performance guarantees by leveraging the structure of the underlying halfspace into algorithmic design. Of central interest is a property termed attribute efficiency, which proves to be useful when the data lie in a high-dimensional space~\citep{littlestone1987learning}, or even in an infinite-dimensional space but with bounded number of effective attributes~\citep{blum1990learning}. In the statistics and signal processing community, it is often referred to as sparsity, dating back to the celebrated Lasso estimator~\citep{tibshirani1996regression,chen1998atomic,candes2005decoding,donoho2006compressed}. Recently, learning of sparse halfspaces in an attribute-efficient manner was highlighted as an open problem in~\citet{feldman2014open}, and in a series of recent works~\citep{plan2013robust,awasthi2016learning,zhang2018efficient,zhang2020efficient}, this property was carefully explored for label-noise-tolerant learning of halfspaces with improved or even near-optimal sample complexity, label complexity, or generalization error, where the key insight is that such structural constraint effectively controls the complexity of the hypothesis class~\citep{zhang2002covering,kakade2008complexity}.

Compared to the rich set of positive results on attribute-efficient learning of sparse halfspaces under label noise, less is known when both instances and labels are corrupted. Specifically, under the $\eta$-malicious noise model~\citep{valiant1985learning,kearns1988learning}, there is an unknown hypothesis $w^*$ and an unknown instance distribution $D$ selected from a certain family by an adversary. Each time with probability $1-\eta$, the adversary returns an instance $x$ drawn from $D$ and the label $y = \sign{w^*\cdot x}$; with probability $\eta$, it instead is allowed to return an arbitrary pair $(x, y) \in \Rd \times \{-1, 1\}$ that may depend on the state of the learning algorithm and the history of its outputs. Since this is a much more challenging noise model, only recently has an algorithm with near-optimal noise tolerance been established in~\citet{awasthi2017power}, although without attribute efficiency. It is worth noting that the problem of learning sparse halfspaces is also closely related to one-bit compressed sensing~\citep{boufounos2008bit} where one is allowed to utilize any distribution $D$  over measurements for recovering the target hypothesis. However, even with such strong condition, existing theory therein can only handle label noise~\citep{plan2013one,awasthi2016learning,baraniuk2017exponential}. This naturally raises two fundamental questions:~1)~can we design attribute-efficient learning algorithms that are capable of tolerating the malicious noise; and 2)~can we still obtain near-optimal performance guarantees on the degree of noise tolerance and on the sample complexity.


In this paper, we answer the two questions in the affirmative under a mild distributional assumption that $D$ is chosen from the family of isotropic log-concave distributions~\citep{lovasz2007geometry,vempala2010random}, which covers prominent distributions such as normal distributions, exponential distributions, and logistic distributions. Moreover, we take label complexity into consideration~\citep{cohn1994improving}, for which we show that our bound is near-optimal in that aspect. We build our algorithm upon the margin-based active learning framework~\citep{balcan2007margin}, which queries the label of an instance when it has small ``margin'' with respect to the currently learned hypothesis.

From a high level, this work can be thought of as extending the best known result of \citet{awasthi2017power} to the high-dimensional regime. However, even under the low-dimensional setting where $s=d$, our bound of label complexity is better than theirs in terms of the dependence on the dimension $d$:~they have a quadratic dependence whereas we have a linear dependence (up to logarithmic factors). 
Moreover, as we will describe in Section~\ref{sec:alg}, obtaining such algorithmic extension is nontrivial both computationally and statistically. This work can also be viewed as an extension of  \citet{zhang2018efficient} to the malicious noise model. In fact, our construction of empirical risk minimization is inspired by that work. However, they considered only label noise which makes their algorithm and analysis not applicable to our setting: it turns out that when facing malicious noise, a sophisticated design of outlier removal paradigm is crucial for optimal noise tolerance~\citep{klivans2009learning}.

Also in line with this work is learning with nasty noise~\citep{diakonikolas2018learning} and robust sparse functional estimation~\citep{balakrishnan2017computation}. Both works considered more general setting in the following sense: \citet{diakonikolas2018learning} showed that by properly adapting the techniques in robust mean estimation, some more general concepts, e.g. low-degree polynomial threshold functions and intersections of halfspaces, can be efficiently learned with $\poly{d, 1/\epsilon}$ sample complexity;  \citet{balakrishnan2017computation} showed that under proper sparsity assumptions, a sample complexity bound of $\poly{s, \log d, 1/\epsilon}$ can be achieved for many sparse estimation problems, such as generalized linear models with Lipschitz mapping functions and covariance estimation. However, we remark that neither of them obtained label efficiency. In addition, when adapted to our setting, Theorem~1.5 of \citet{diakonikolas2018learning} only handles noise rate $\eta \leq O(\epsilon^{c})$ for some constant $c$ that is greater than one, while as to be shown in Section~\ref{sec:guarantee}, we obtain the near-optimal noise tolerance $\eta \leq O(\epsilon)$. \cite{balakrishnan2017computation} achieved near-optimal noise tolerance but their analysis is restricted to the Gaussian marginal distribution and Lipschitz mapping functions. In addition to such fundamental differences, the main techniques we develop are distinct from theirs, which will be described in more detail in Section~\ref{subsec:comp}.


\subsection{Main results}

We informally present our main results below; readers are referred to Theorem~\ref{thm:main} in Section~\ref{sec:guarantee} for a precise statement.

\begin{theorem}[Informal]
Consider the malicious noise model with noise rate $\eta$. If the unlabeled data distribution is isotropic log-concave and the underlying halfspace $w^*$ is $s$-sparse, then there is an algorithm that for any given target error rate $\epsilon \in (0, 1)$, PAC learns the underlying halfspace in polynomial time provided that $\eta \leq O(\epsilon)$. In addition, the label complexity is $\tilde{O}\big(s \log^4 \frac{d}{\epsilon}\big)$ and the sample complexity is $\tilde{O}\del[1]{\frac{1}{\epsilon}s^2 \log^5 d}$.
\end{theorem}

First of all, note that the noise tolerance is near-optimal as \cite{kearns1988learning} showed that a noise rate greater than $\frac{\epsilon}{1+\epsilon}$ cannot be tolerated by any algorithm regardless of the computational power. The following fact establishes the optimality of our label complexity.

\begin{lemma}
Active learning of $s$-sparse halfspaces under isotropic log-concave distributions in the realizable case has an information-theoretic label complexity lower bound of $\Omega\del[1]{ s (\log\frac1\epsilon + \log \frac{d}{s}) }$.
\end{lemma}
To see this lemma, observe that there exist $\epsilon$-packings of $s$-sparse halfspaces with sizes $(\frac1\epsilon)^{\Omega(s)}$~\citep{long1995sample} and $(\frac{d}{s})^{\Omega(s)}$~\citep{raskutti2011minimax}; applying Theorem~1 of~\citet{kulkarni1993active} gives the lower bound.

\subsection{Related works}

\citet{kearns1988learning} presented a general analysis on efficiently learning halfspaces, showing that even without any distributional assumptions, it is possible to tolerate the malicious noise at a rate of ${\Omega}(\epsilon/d)$, but a noise rate greater than $\frac{\epsilon}{1+\epsilon}$ cannot be tolerated. The noise model was further studied by \citet{schapire1992design,bshouty1998new,cesa1999sample}, and \citet{kalai2005agnostic} obtained a noise tolerance $\Omega(\epsilon/d^{1/4})$ when $D$ is the uniform distribution. \citet{klivans2009learning} improved this result to $\Omega(\epsilon^2/\log(d/\epsilon))$ for the uniform distribution, and showed a noise tolerance $\Omega(\epsilon^3/\log^2(d/\epsilon))$ for isotropic log-concave distributions. A near-optimal result of $\Omega(\epsilon)$ was established in \citet{awasthi2017power} for both uniform and isotropic log-concave distributions.

Achieving attribute efficiency has been a long-standing goal in machine learning and statistics~\citep{blum1990learning,blum1995learning}, and has found a variety of applications with strong theoretical backend. A partial list includes online classification~\citep{littlestone1987learning}, learning decision lists~\citep{servedio1999computational,klivans2004toward,long2006attribute}, compressed sensing~\citep{donoho2006compressed,candes2008introduction,tropp2010computational,shen2018tight},  one-bit compressed sensing~\citep{boufounos2008bit,plan2016generalized}, and variable selection~\citep{fan2001variable,fan2008high,shen2017iteration,shen2017partial}.

Label-efficient learning has also been broadly studied since gathering high quality labels is often expensive. The prominent approaches include disagreement-based active learning~\citep{hanneke2011rates,hanneke2014theory}, margin-based active learning~\citep{balcan2007margin,balcan2013active,yan2017revisiting}, selective sampling~\citep{cavallanti2011learning,dekel2012selective}, and adaptive one-bit compressed sensing~\citep{zhang2014efficient,baraniuk2017exponential}. There are also a number of interesting works that appeal to extra information to mitigate the labeling cost, such as comparison~\citep{xu2017noise,kane2017active} and search~\citep{balcan2012robust,beygelzimer2016search}.

Recent works such as \citet{diakonikolas2016robust,lai2016agnostic} studied mean estimation under a strong noise model where in addition to returning dirty instances, the adversary has also the power of eliminating a few clean instances, similar to the nasty noise model in learning halfspaces~\citep{bshouty2002pac}. The main technique of robust mean estimation is a novel outlier removal paradigm, which uses the spectral norm of the covariance matrix to detect dirty instances. This is similar in spirit to the idea of \citet{klivans2009learning,awasthi2017power} and the current work. However, there is no direct connection between mean estimation and halfspace learning since the former is an unsupervised problem while the latter is supervised (although any connection would be very interesting). Very recently, such technique was extensively investigated in a variety of problems such as clustering and linear regression; we refer the reader to a comprehensive survey by~\citet{diakonikolas2019recent} for more information.

\paragraph{Roadmap.}
We collect useful notations and formally define the problem in Section~\ref{sec:prelim}. In Section~\ref{sec:alg}, we describe our algorithms, followed by a theoretical analysis in Section~\ref{sec:guarantee}. We conclude this paper in Section~\ref{sec:conclusion}, and defer all proof details to the appendix.

\section{Preliminaries}\label{sec:prelim}

We study the problem of learning sparse halfspaces in $\Rd$ under the malicious noise model with noise rate $\eta \in [0, 1/2)$~\citep{valiant1985learning,kearns1988learning}, where an oracle $\oraclexy$ (i.e. adversary) first selects a member $D$ from a family of distributions $\calD$ and a concept $w^*$ from a concept class $\calC$; during the learning process, $D$ and $w^*$ are fixed. Each time the adversary is called, with probability $1-\eta$, a random pair $(x, y)$ is returned to the learner with $x \sim D$ and $y = \sign{w^* \cdot x}$, referred to as a clean sample; with probability $\eta$, the adversary can return an {\em arbitrary} pair $(x, y) \in \Rd \times \{-1, 1\}$, referred to as a dirty sample. The adversary is assumed to have unrestricted computational power to search dirty samples that may depend on, e.g. the states of the learning algorithm and the history of its outputs. Formally, we make the following distributional assumptions.

\begin{assumption}\label{as:x}
Let $\calD$ be the family of isotropic log-concave distributions. The underlying distribution $D$ from which clean instances are drawn is chosen from $\calD$ by the adversary, and is fixed during the learning process. The learner is given the knowledge of $\calD$ but not of $D$.
\end{assumption}
\begin{assumption}\label{as:y}
With probability $1-\eta$, the adversary returns a pair $(x, y)$ where $x \sim D$ and $y =\sign{w^* \cdot x}$; with probability $\eta$, it may return an arbitrary pair $(x, y) \in \Rd \times \{-1, 1\}$.
\end{assumption}

Since we are interested in obtaining a label-efficient algorithm, we will consider a natural extension of such passive learning model. In particular, \citet{awasthi2017power} proposed to consider the following:~when a labeled instance $(x, y)$ is generated, the learner only has access to an instance-generation oracle $\oraclex$ which returns $x$, and must make a separate call to a label revealing oracle $\oracley$ to obtain $y$. We refer to the total number of calls to $\oraclex$ as the sample complexity of the learning algorithm, and to that of $\oracley$ as the label complexity.

We will presume that the concept class $\calC$ consists of homogeneous halfspaces that have unit $\ell_2$-norm and are $s$-sparse, i.e. the number of non-zero elements of any $w \in \calC$ is at most $s$ where $s \in \{1, 2, \dots, d\}$. The learning algorithm is given this concept class, that is, the set of homogeneous $s$-sparse halfspaces. For a hypothesis $w \in \calC$, we define its error rate on a distribution ${D}$ as $\err_{{D}}(w) = \Pr_{x \sim {D}}\del{\sign{w\cdot x} \neq \sign{w^* \cdot x}}$. The goal of the learner is to find a hypothesis $w$ in polynomial time such that with probability $1-\delta$, $\err_D(w) \leq \epsilon$ for any given failure confidence $\delta \in (0, 1)$ and any error rate $\epsilon \in (0, 1)$, with a few calls to $\oraclex$ and $\oracley$.

For a reference vector $u \in \Rd$ and a positive scalar $b$, we call the region $X_{u, b} \defeq \{x \in \Rd: \abs{u\cdot x} \leq b\}$ as band, and we denote by $D_{u, b}$ the distribution obtained by conditioning $D$ on the event $x \in X_{u, b}$. Given a hypothesis $w$ in $\Rd$, a labeled instance $(x, y)$, and a parameter $\tau>0$, we define the $\tau$-hinge loss $\ell_{\tau}(w; x, y) = \max\big\{0, 1 - \frac{1}{\tau}y (w\cdot x)\big\}$. For a labeled set $S=\{ (x_i, y_i) \}_{i=1}^n$, we define $\ell_{\tau}(w; S) = \frac{1}{n} \sum_{i=1}^{n} \ell_{\tau}(w; x_i, y_i)$.

For $p \geq 1$, we denote by $B_p(u, r)$ the $\ell_p$-ball centering at the point $u$ with radius $r > 0$, i.e. $B_p(u, r) = \{w \in \Rd: \norm{w-u}_p \leq r \}$. We will be particularly interested in the cases $p=1, 2, \infty$. For a vector $u \in \Rd$, the hard thresholding operation $\calH_s(u)$ keeps its $s$ largest (in absolute value) elements and sets the remaining to zero. Let $u, v \in \Rd$ be two vectors; we write $\theta(u, v)$ to denote the angle between them, and write $u \cdot v$ to denote their inner product. For a matrix $H$, we denote by $\nuclearnorm{H}$ its trace norm (also known as the nuclear norm), i.e. the sum of its singular values. We will also use $\onenorm{H}$ to denote the entrywise $\ell_1$-norm of $H$, i.e. the sum of absolute values of its entries. 
If $H$ is a symmetric matrix, we use $H \succeq 0$ to denote that it is positive semidefinite.

Throughout this paper, the subscript variants of the lowercase letter $c$, e.g. $c_1$ and $c_2$, are reserved for specific absolute constants that are uniquely determined by the distribution $D$. We also reserve $C_1$ and $C_2$ for specific constants. We remark that the value of all the constants involved in the paper does not depend on the underlying distribution $D$, but rather on the knowledge of $\calD$ given to the learner. We collect all the definitions of these constants in Appendix~\ref{sec:app:constants}.

\section{Main Algorithm}\label{sec:alg}

We first present an overview of our learning algorithm, followed by specifying all the hyper-parameters used therein. Then we describe in detail the attribute-efficient outlier removal scheme, which is the core technique in the paper.

\subsection{Overview}

Our main algorithm, namely Algorithm~\ref{alg:main}, is based on the celebrated margin-based active learning framework~\citep{balcan2007margin}. The key observation is that a good classifier can be learned by concentrating on fitting only the most informative labeled instances, measured by the closeness to the current decision boundary (i.e. the closer the more informative). In our algorithm, the sampling region is
set as $\Rd$ at phase $k = 1$, and is set as the band $X_{w_{k-1}, b_k} = \{x \in \Rd: \abs{w_{k-1} \cdot x} \leq b_k\}$ at phases $k \geq 2$. Once we obtain the instance set $\bar{T}$, we perform a pruning step that removes all instances having large $\ell_{\infty}$-norm. This is motivated by our analysis that with high probability, all clean instances in $\bar{T}$ must have small $\ell_{\infty}$-norm provided that Assumption~\ref{as:x} is satisfied. Since the oracle $\oraclex$ may output dirty instances, we design an attribute-efficient soft outlier removal procedure, which aims to find proper weights for all instances in $T$, such that the clean instances (i.e. those from $D_{w_{k-1}, b_k}$) have overwhelming weights compared to dirty instances. Equipped with the learned weights, it is possible to minimize the reweighted hinge loss to obtain a refined halfspace. However, this would lead to a suboptimal label complexity since we have to query the label for all instances in $T$. Our remedy is to randomly sample a few points from $T$ according to their importance, which is crucial for us to obtain near-optimal label complexity.

When minimizing the hinge loss, we carefully construct the constraint set $W_k$ with three properties. First, it has an $\ell_2$-norm constraint. As a useful fact of isotropic log-concave distributions, the $\ell_2$-distance to the underlying halfspace $w^*$ is of the same order as the error rate. Thus, if we were able to ensure that the target halfspace $w^*$ stays in $W_k$, we would show that the error rate of $w_k$ is as small as $O(r_k)$, the radius of the $\ell_2$-ball. Second, $W_k$ has an $\ell_1$-norm constraint, which is well-known for its power to promote sparse solutions and to guarantee attribute-efficient sample complexity~\citep{tibshirani1996regression,chen1998atomic,candes2005decoding,plan2013robust}. Lastly, the $\ell_2$ and $\ell_1$ radii of $W_k$ shrinks by a constant factor in each phase; hence, when Algorithm~\ref{alg:main} terminates, the radius of the $\ell_2$-ball will be as small as $O(\epsilon)$. Notably, \citet{zhang2018efficient} also utilizes such constraint for active learning of sparse halfspaces, but only under the setting of label noise.

The last step in Algorithm~\ref{alg:main} is to perform hard-thresholding $\calH_s$ on the solution $v_k$ followed by $\ell_2$-normalization. Roughly speaking, these two steps will produce an iterate $w_k$ consistent with the structure of $w^*$ (i.e. $w_k$ is guaranteed to belong to the concept class $\calC$), and more importantly, will be useful to show that $w^*$ lies in $W_k$ in all phases.

\begin{algorithm}[t]
\caption{Attribute and Label-Efficient Algorithm Tolerating Malicious Noise}
\label{alg:main}
\begin{algorithmic}[1]
\REQUIRE Error rate $\epsilon$, failure probability $\delta$, sparsity parameter $s$, an instance generation oracle $\oraclex$, a label revealing oracle $\oracley$.
\ENSURE A halfspace $w_{k_0}$ such that  $\err_D(w_{k_0}) \leq \epsilon$ with probability $1-\delta$.
\STATE $k_0 \gets \big\lceil \log\del[1]{\frac{\pi}{16 c_1 \epsilon}} \big\rceil$.
\STATE Initialize $w_0$ as the zero vector in $\Rd$.
\FOR{phases $k = 1, 2, \dots, k_0$}
\STATE Clear the working set $\bar{T}$.
\STATE  If $k=1$, independently draw $n_k$ instances from $\text{EX}^x_{\eta}(D, w^*)$ and put them into $\bar{T}$; otherwise, draw $n_k$ instances from $\text{EX}^x_{\eta}(D, w^*)$ conditioned on $\abs{w_{k-1} \cdot x} \leq b_k$ and put into $\bar{T}$.
\label{step:draw-unl-samples}
\STATE {\bfseries Pruning:} Remove all instances $x$ in $\bar{T}$ with $\infnorm{x} > c_9 \log\frac{48 n_k d}{b_k \delta_k}$ to form a set $T$.
\label{step:remove-ell-infty}

\STATE {\bfseries Soft outlier removal:} Apply Algorithm~\ref{alg:reweight} to $T$ with $u \gets w_{k-1}$, $b \gets b_k$, $r \gets r_k$, $\rho \gets \rho_k$, $\xi \gets \xi_k$, $C \gets 2 C_2$, and let $q = \cbr{q(x)}_{x \in T}$ be the returned function. Normalize $q$ to form a probability distribution $p$ over $T$.

\STATE {\bfseries Random sampling:} $S_k \gets$ Independently draw $m_k$ instances (with replacement) from $T$ according to $p$ and query $\text{EX}^y_{\eta}(D, w^*)$ for their labels. 
\label{step:random-sampling}

\STATE Let $W_k = B_2(w_{k-1}, r_k) \cap B_1(w_{k-1}, \rho_k)$. Find $v_k \in W_k$ such that
\begin{equation*}
\ell_{\tau_k}(v_k; S_k) \leq \min_{w \in W_k} \ell_{\tau_k}(w; S_k) + \kappa.
\end{equation*}
\vspace{-0.1in}
\STATE $w_k \gets \frac{\calH_s(v_k)}{\twonorm{\calH_s(v_k)}}$.
\ENDFOR

\RETURN $w_{k_0}$.
\end{algorithmic}
\end{algorithm}

\subsection{Hyper-parameter setting}\label{subsec:param-setting}
We elaborate on our hyper-parameter setting that is used in Algorithm~\ref{alg:main} and our analysis. Let $g(t) = c_2 \del[1]{ 2 t \exp(-t) + \frac{c_3 \pi}{4} \exp\del[1]{-\frac{c_4 t}{4\pi}}  + 16 \exp(-t) }$, where the constants are specified in Appendix~\ref{sec:app:constants}. Observe that there exists an absolute constant $\bar{c} \geq 8\pi / c_4$ satisfying $g(\bar{c}) \leq 2^{-8}\pi$, since the continuous function $g(t) \to 0$ as $t \to +\infty$ and all the involved quantities in $g(t)$ are absolute constants. Given such constant $\bar{c}$, we set $b_k = \bar{c} \cdot 2^{-k-3}$, $\tau_k = c_0 \kappa \cdot  \min\{b_k, 1/9\}$, $\delta_k = \frac{\delta}{(k+1)(k+2)}$,
\begin{equation*}
r_k = \begin{cases}
1, &k=1\\
2^{-k-3}, &k \geq 2
\end{cases},\ \text{and} \ 
\rho_k = \begin{cases}
\sqrt{s}, &k=1\\
\sqrt{2s} \cdot 2^{-k-3}, &k \geq 2
\end{cases}.
\end{equation*}
We set the constant $\kappa = \exp(-\bar{c})$, and choose $\xi_k = \min\big\{\frac{1}{2}, \frac{\kappa^2}{16} \del[1]{1 + 4\sqrt{C_2} {z_k}/{\tau_k} }^{-2}\big\}$. Observe that all $\xi_k$'s are lower bounded by the constant $c_6 \defeq \min\Big\{\frac{1}{2}, \frac{\kappa^2}{16} \del[2]{ 1 + \frac{4 }{c_0 \kappa \bar{c}} \sqrt{C_2 \bar{c}^2 + C_2} }^{-2} \Big\}$. Our theoretical guarantee holds for any noise rate $\eta \leq c_5 \epsilon$, where the constant $c_5 := \frac{c_8}{2\pi}\bar{c} c_1 c_6$.

We set the total number of phases $k_0 = \big\lceil \log\del[1]{\frac{\pi}{16 c_1 \epsilon}} \big\rceil$ in Algorithm~\ref{alg:main}. Consider any phase $k \geq 1$. We use $n_k = \tilde{O}\del[1]{ s^2 \log^4 \frac{d}{b_k} \cdot \del[1]{\log d + {\log^3\frac{1}{\delta_k}} }}$ as the size of unlabeled instance set $\bar{T}$. We will show that by making $N_k = O\del{ n_k / b_k}$ calls to $\oraclex$, Algorithm~\ref{alg:main} is guaranteed to obtain such $\bar{T}$ in each phase with high probability. We set $m_k = \tilde{O}\del[1]{s \log^2\frac{d}{b_k \delta_k} \cdot \log\frac{d}{\delta_k}}$ as the size of labeled instance set $S_k$, which is also the number of calls to $\oracley$. Note that $N \defeq \sum_{k=1}^{k_0}N_k$ is the sample complexity of Algorithm~\ref{alg:main}, and $m \defeq \sum_{k=1}^{k_0}m_k$ is its label complexity.


\subsection{Attribute and computationally efficient soft outlier removal}

Our soft outlier removal procedure is inspired by \citet{awasthi2017power}. We first briefly describe their main idea. Then we introduce a natural extension of their approach to the high-dimensional regime and show why it fails. Lastly, we present our novel outlier removal scheme.

To ease our discussion, we decompose $T = \TC \cup \TD$ where $\TC$ is the set of clean instances in $T$ and $\TD$ consists of all dirty instances. Ideally, we would expect to find a function $q: T \rightarrow [0, 1]$ such that $q(x) = 1$ for all $x \in \TC$ and $q(x) = 0$ otherwise. Suppose that $\xi$ is the fraction of dirty instances in $T$. Then one would expect that the total weights $\sum_{x \in T} q(x)$ is as large as $(1-\xi)\abs{T}$ in order to include such ideal function. On the other hand, we must restrict the weights of dirty instances; namely, we need to characterize under what conditions $\TC$ can be distinguished from $\TD$. The key observation made in \citet{klivans2009learning} and \citet{awasthi2017power} is that if the dirty instances want to deteriorate the hinge loss (which is the purpose of the adversary), they must lead to a variance\footnote{We follow~\citet{awasthi2017power} and slightly abuse the word ``variance'' without subtracting the squared mean of $w\cdot x$.} of $w\cdot x$ orders of magnitude larger than $\Omega(b^2 + r^2)$ on the direction of a particular halfspace. Thus, it suffices to find a proper weight for each instance, such that the reweighted variance $\frac{1}{\abs{T}}\sum_{x \in T} q(x) (w\cdot x)^2$ is as small as $O(b^2 + r^2)$ for all feasible halfspaces $w \in W$. Now it remains to resolve two questions:~1)~how many instances do we need to draw in order to guarantee the existence of such function $q$; and 2)~how to find a feasible function $q$ in polynomial time.

If label complexity were our only objective, we could have used the soft outlier removal procedure of~\citet{awasthi2017power} directly, i.e. we set $W = B_2(u, r)$, which in conjunction with the $\ell_1$-norm constrained hinge loss minimization of \cite{zhang2018efficient} would result in an $\tilde{O}\del[1]{\frac{d^2}{\epsilon}}$ sample complexity and a $\poly{s, \log d, \log(1/\epsilon)}$ label complexity. However, as we would also like to optimize for the learner's sample complexity by utilizing the sparsity assumption, we need an attribute-efficient outlier removal procedure.

\subsubsection{A natural approach and why it fails}

It is well-known that incorporating an $\ell_1$-norm constraint often leads to a sample complexity sublinear in the dimension~\citep{zhang2002covering,kakade2008complexity}. Thus, a natural approach for attribute-efficient outlier removal is to set $W = B_2(u, r) \cap B_1(u, \rho)$ for some carefully chosen radius $\rho > 0$. With the new localized concept space, it is possible to show that a sample size of $\poly{s, \log d}$ suffices to guarantee the existence of a function $q$ such that the reweighted variance is small over all $w \in W$. However, on the computational side, for a given $q$, we will have to check the reweighted variance for all $w \in W$, which amounts to finding a global optimum of the following program:
\begin{equation}\label{eq:spca}
\max_{w \in \Rd} \ \frac{1}{\abs{T}}\sum_{x \in T} q(x) (w \cdot x)^2,\ \st\ \twonorm{w - u} \leq r,\ \onenorm{w - u} \leq \rho.
\end{equation}
The above program is closely related to the problem of sparse principal component analysis~(PCA)~\citep{zou2006sparse}, and unfortunately it is known that finding a global optimum is NP-hard~\citep{steinberg2005computation,tillmann2014computational}.

\begin{algorithm}[t]
\caption{Attribute-Efficient Localized Soft Outlier Removal}
\label{alg:reweight}
\begin{algorithmic}[1]
\REQUIRE{Reference vector $u$, band width $b$, radius $r$ for $\ell_2$-ball, radius $\rho$ for $\ell_1$-ball, empirical noise rate $\xi$, absolute constant $C$, a set of unlabeled instances $T$ where for all $x \in T$, $\abs{u \cdot x} \leq b$.}
\ENSURE{A function $q: T \rightarrow [0, 1]$.}
\STATE Define the convex set of matrices $\calM = \big\{ H \in \Rdd:\  H \succeq 0,\ \nuclearnorm{H} \leq r^2,\ \onenorm{H} \leq \rho^2\big\}$.

\STATE Find a function $q: T \rightarrow [0, 1]$ satisfying the following constraints:
\begin{enumerate}
\item \label{item:alg2:0-1}  for all $x \in T, 0 \leq q(x) \leq 1$;

\item \label{item:alg2:err} $\sum_{x \in T} q(x) \geq (1 - \xi)\abs{T}$;

\item \label{item:alg2:var} $\sup_{H \in \calM} \frac{1}{\abs{T}} \sum_{x \in T} q(x)  x\trans H x \leq C (b^2 + r^2)$.
\end{enumerate}

\RETURN $q$.
\end{algorithmic}
\end{algorithm}

\subsubsection{Convex relaxation of sparse principal component analysis}

Our goal is to find a function $q$ such that the objective value in \eqref{eq:spca} is less than $O\del{b^2 + r^2}$ for all $w \in W$. To circumvent the computational intractability caused by the non-convexity of the objective function, we consider an alternative formulation using semidefinite programming~(SDP), similar to the approach of~\citet{dAspremont2007direct}. First, let $v = w- u$. It is not hard to see that $(w \cdot x)^2 \leq 2 (u \cdot x)^2 + 2(v \cdot x)^2$. Due to our localized sampling scheme, we have $(u \cdot x)^2 \leq b^2$ with probability $1$. Thus, we only need to examine the maximum value of $\frac{1}{\abs{T}}\sum_{x \in T} q(x)(v\cdot x)^2$ over $v \in B_2(0, r) \cap B_1(0, \rho)$. Now the technique of \citet{dAspremont2007direct} comes in: the rank-one symmetric matrix $vv\trans$ is replaced by a new variable $H \in \Rdd$ which is positive semidefinite, and the vector $\ell_2$ and $\ell_1$-norm constraints are relaxed to the matrix trace and $\ell_1$-norm constraints respectively as follows:
\begin{equation}\label{eq:spca-relax}
\max_{H \in \Rdd} \frac{1}{\abs{T}} \sum_{x \in T} q(x) x\trans H x,\ \st\ H \succeq 0,\ \nuclearnorm{H} \leq r^2,\ \onenorm{H} \leq \rho^2.
\end{equation}
The program~\eqref{eq:spca-relax} has two salient features: first, it is a semidefinite program that can be optimized efficiently~\citep{boyd2004convex}; second, if its objective value is upper bounded by $O\del{b^2 + r^2}$, we immediately obtain that the reweighted variance is well controlled. This is the theme of the following lemma.

\begin{lemma}\label{lem:xHx-emp}
Suppose that Assumption~\ref{as:x} and \ref{as:y} are satisfied, and that $\eta \leq c_5 \epsilon$.  There exists a constant $C_2 > 2$ such that the following holds. For any phase $k$ of Algorithm~\ref{alg:main} with $1 \leq k \leq k_0$, write $\calM_k = \cbr{ H \in \Rdd:\  H \succeq 0,\ \nuclearnorm{H} \leq r_k^2,\ \onenorm{H} \leq \rho_k^2}$. Then with probability $1 - \frac{\delta_k}{24}$ over the draw of $\TC$, we have
\begin{equation*}
\sup_{H \in \calM_k}\frac{1}{\abs{\TC}} \sum_{x \in \TC} x\trans H x \leq 2 C_2 (b_k^2 + r_k^2),
\end{equation*}
provided that $\abs{\TC} \geq \tilde{O}\del[2]{s^2 \log^4\frac{d}{b_k} \cdot \del[1]{\log d + \log^2 \frac{1}{\delta_k}}}$.
\end{lemma}

Recall that Algorithm~\ref{alg:main} sets $n_k = \tilde{O}\del[1]{ s^2 \log^4 \frac{d}{b_k} \cdot \del[1]{\log d + {\log^3\frac{1}{\delta_k}} }}$, which suffices to guarantee the condition on $\abs{\TC}$ holds (see Appendix~\ref{sec:app:prune}); therefore, the above concentration bound holds with high probability.
As a result, it is not hard to verify that the function $q: T \rightarrow [0, 1]$, where $q(x) = 1$ for all $x \in \TC$ and $q(x) = 0$ for all $x \in \TD$, satisfies all three constraints in Algorithm~\ref{alg:reweight}. In other words, Lemma~\ref{lem:xHx-emp} establishes the existence of a feasible function $q$ to Algorithm~\ref{alg:reweight}. Furthermore, observe that the optimization problem of finding a feasible $q$ in Algorithm~\ref{alg:reweight} is a semi-infinite linear program. For a given candidate $q$, we can construct an efficient oracle as follows: it checks if $q$ violates the first two constraints; if not, it checks the last constraint by invoking a polynomial-time SDP solver to find the maximum objective value of \eqref{eq:spca-relax}. It is well-known that equipped with such separation oracle, Algorithm~\ref{alg:reweight} will return a desired function $q$ in polynomial time by the ellipsoid method~\cite[Chapter 3]{grotschel2012geometric}.

\subsubsection{Comparison to prior works}\label{subsec:comp}

We remark that the setting of $n_k$ results in a sample complexity of $\tilde{O}\del[1]{\frac{s^2}{b_k}}$ for phase $k$ (see a formal statement in Lemma~\ref{lem:N}), which implies a total sample complexity of $\tilde{O}\del[1]{\frac{s^2}{\epsilon}}$. When $s \ll d$, this  substantially improves upon the sample complexity of $\tilde{O}\del[1]{\frac{d^2}{\epsilon}}$ when naively applying the soft outlier removal procedure in~\citet{awasthi2017power}.

We remark three crucial technical differences from \citet{diakonikolas2018learning} and \citet{balakrishnan2017computation}. First, we progressively restrict the variance to identify dirty instances, i.e. the variance upper bound is set as $O(1)$ at the beginning of Algorithm~\ref{alg:main} and progressively decreases to $O(\epsilon^2)$ (see our setting of $b_k$ and $r_k$), 
while in \citet{diakonikolas2018learning,balakrishnan2017computation} and many of their follow-up works it is typically fixed to $O(\epsilon)$. Second, we control the variance locally, i.e. we only require a small variance over a localized instance space $D_{w_{k-1}, b_k}$ and a localized concept space $\calM_k$. Third, the small variance is used to robustly estimate the hinge loss in our work, while in \citet{diakonikolas2018learning} it was utilized to approximate the Chow parameters. All these problem-specific design of outlier removal are vital for us to obtain the first near-optimal guarantee on attribute efficiency and label efficiency for learning sparse halfspaces. 

\section{Performance Guarantee}\label{sec:guarantee}

In the following, we always presume that the underlying halfspace is parameterized by $w^*$, which is $s$-sparse and has unit $\ell_2$-norm. This condition may not be explicitly stated in our analysis.

Our main theorem is as follows. We note that there are two sources of randomness in Algorithm~\ref{alg:main}: the random draw of instances from $\oraclex$, and the random sampling step (i.e. Step~\ref{step:random-sampling}); the probability is taken over all the randomness in the algorithm.

\begin{theorem}\label{thm:main}
Suppose that Assumptions~\ref{as:x} and~\ref{as:y} are satisfied. There exists an absolute constant $c_5$ such that for any $\epsilon \in (0, 1)$ and $\delta \in (0, 1)$, if $\eta \leq c_5 \epsilon$, then with probability at least $1 - \delta$, $\err_D(w_{k_0}) \leq \epsilon$ where $w_{k_0}$ is the output of Algorithm~\ref{alg:main}. Furthermore, Algorithm~\ref{alg:main} has a sample complexity of $\tilde{O}\big(\frac{1}{\epsilon} s^2 \log^4 d \cdot \del[1]{\log d + \log^3 \frac{1}{\delta}}\big)$, and a  label complexity of $\tilde{O}\big(s \log^2\frac{d}{\epsilon\delta} \cdot \log\frac{d}{\delta} \cdot \log\frac{1}{\epsilon} \big)$, and has running time $\poly{d, {1}/{\epsilon}, {1}/{\delta}}$.
\end{theorem}

Algorithm~\ref{alg:main} can be straightforwardly modified to work in the passive learning setting, where the learner has direct access to the labeled instance oracle $\oraclexy$. The modified algorithm works as follows: it calls $\oraclexy$ to obtain a pair of instance and the label whenever Algorithm~\ref{alg:main} calls $\oraclex$. In particular, for the passive learning algorithm, the working set $\bar{T}$ is always a labeled instance set, and there is no need for it to query $\oracley$ in the random sampling step.

We have the following simple corollary which is an immediate result from Theorem~\ref{thm:main}.
\begin{corollary}\label{cor:main}
Suppose that Assumptions~\ref{as:x} and~\ref{as:y} are satisfied. There exists a polynomial-time algorithm (that has access to only $\oraclexy$) and an absolute constant $c_5$ such that for any $\epsilon \in (0, 1)$ and $\delta \in (0, 1)$, if $\eta \leq c_5 \epsilon$, then with probability at least $1 - \delta$, the algorithm outputs a hypothesis with error at most $\epsilon$, using $\tilde{O}\big(\frac{1}{\epsilon} s^2 \log^4 d \cdot \del[1]{\log d + \log^3 \frac{1}{\delta}}\big)$ labeled instances.
\end{corollary}

We need an ensemble of new results to prove Theorem~\ref{thm:main}. Specifically, we propose new techniques to control the sample and computational complexity of soft outlier removal, and a new analysis of label complexity by making full use of the localization in the instance and concept spaces. We elaborate on them in the following, and sketch the proof of Theorem~\ref{thm:main} at the end of this section.

\subsection{Localized sampling in the instance space}\label{subsec:local-sampling}

Localized sampling, also known as margin-based active learning, is a useful technique proposed in \citet{balcan2007margin}. Interestingly, under isotropic log-concave distributions, \citet{balcan2013active} showed that if the band width $b$ is large enough, the region outside the band, i.e. $\{x \in \Rd: \abs{w \cdot x} > b\}$, can be safely ``ignored'', in the sense that, if $w$ is close enough to $w^*$, it is guaranteed to incur a small error rate therein. Motivated by this elegant finding, theoretical analyses in the literature are often dedicated to bounding the error rate within the band, and it is now well understood that a constant error rate within the band suffices to ensure significant progress in each phase~\citep{awasthi2015efficient,awasthi2017power,zhang2018efficient}. We follow this line of reasoning and our technical contribution is to show how to obtain such constant error rate with near-optimal label complexity and noise tolerance.

Our analysis will rely on the condition that $\bar{T}$ has sufficiently many instances. Specifically, in order to collect $n_k$ instances to form the working set $\bar{T}$, we need to call $\oraclex$ enough number of times since our sampling is localized within the band $X_k \defeq \cbr{x: \abs{w_{k-1} \cdot x} \leq b_k}$. The following lemma characterizes the sample complexity at phase $k$.

\begin{lemma}\label{lem:N}
Suppose that Assumption~\ref{as:x} and \ref{as:y} are satisfied. Further assume $\eta < \frac{1}{2}$. With probability $1- \frac{\delta_k}{4}$, we will obtain $n_k$ instances that fall into the band $X_k = \{x: \abs{w_{k-1} \cdot x} \leq b_k\}$ by making a number of $N_k = O\del[1]{\frac{1}{ b_k}\del[1]{ n_k + \log\frac{1}{\delta_k} }}$ calls to the instance generation oracle $\oraclex$.
\end{lemma}

\subsection{Attribute and computationally efficient soft outlier removal}\label{subsec:instance-reweight}


 We summarize the performance guarantee of Algorithm~\ref{alg:reweight} in the following proposition.

\begin{proposition}\label{prop:outlier-guarantee}
Consider phase $k$ of Algorithm~\ref{alg:main} for any $1 \leq k \leq k_0$. Suppose that Assumption~\ref{as:x} and \ref{as:y} are satisfied, and that $\eta \leq c_5 \epsilon$. With the setting of $n_k$, with probability $1- \frac{\delta_k}{8}$ over the draw of $\bar{T}$, Algorithm~\ref{alg:reweight} will output a function $q: T \rightarrow [0, 1]$ in polynomial time with the following properties: (1) $\frac{1}{\abs{T}}\sum_{x \in T} q(x) \geq 1 - \xi_k$; (2) for all $w \in W_k$, $\frac{1}{\abs{T}} \sum_{x \in T} q(x) (w\cdot x)^2 \leq 5C_2 \del{b_k^2 + r_k^2}$.
\end{proposition}




Again, we remind that the key difference between our algorithm and that of~\citet{awasthi2017power} is in Constraint~\ref{item:alg2:var} of Algorithm~\ref{alg:reweight}: we require that the ``variance proxy''
$\sum_{x \in T} q(x) x\trans H x$ of the reweighted instances are small for all positive semidefinite $H$ that lies in an intersection of a trace-norm ball and an $\ell_1$-norm ball.
On the statistical side, this favorable constraint set of $H$, in conjunction with Adamczak's bound in empirical processes literature~\citep{adamczak2008tail}, results in sufficient uniform concentration of the variance proxy $x\trans H x$ with a sample complexity of $\poly{s, \log d}$. This significantly improves the sample complexity of $\poly{d}$ established in \citet{awasthi2017power}. The detailed proof can be found in Appendix~\ref{sec:app:instance-reweight}.

\begin{remark}\label{rmk:L1-norm}
While in some standard settings, a proper $\ell_1$-norm constraint suffices to guarantee a desired bound of sample complexity in the high-dimensional regime~\citep{wainwright2009sharp,kakade2008complexity}, we note that in order to establish near-optimal noise tolerance, the $\ell_2$-norm constraint of $w$ (hence the trace-norm of $H$)  is vital as well. Though eliminating it eases the search of a feasible function $q$, this leads to a suboptimal noise tolerance $\eta \leq \Omega({\epsilon}/{s})$. Informally speaking, the per-phase error rate, expected to be a constant, is inherently proportional to the variance $(w\cdot x)^2$ times $\xi_k$, the noise rate within the band. Now without the trace-norm constraint, the variance would be $s$ times larger than before (since we now have to use $\rho_k^2 = O(s r_k^2)$ as a proxy for the constraint set's radius, measured in trace norm). This implies that we need to set $\xi_k$ a factor $1/s$ of before, which in turn indicates that the noise tolerance $\eta$ becomes a factor $1/s$ of before since $\eta /\epsilon \approx \xi_k$. We refer the reader to Proposition~\ref{prop:l(TC)=l(p)-restate} and Lemma~\ref{lem:feasible-to-angle} for details.
\end{remark}

\begin{remark}\label{rmk:s^2}
The quantity $n_k$ has a quadratic dependence on the sparsity parameter $s$. This cannot be improved in some sparse PCA related problems~\citep{berthet2013complexity}, 
but it is not clear whether such dependence is optimal in our case. We leave this investigation to future work.
\end{remark}

Next, we describe the statistical property of the distribution $p$ (obtained by normalizing $q$ returned by Algorithm~\ref{alg:reweight}). Observe that the noise rate within the band is at most $\eta/b_k \leq O(\eta / \epsilon) \leq \xi_k$ since the probability mass of the band is $\Theta(b_k)$~--~an important property of isotropic log-concave distributions. Also, it is possible to show that the variance of clean instances on directions $H \in \calM_k$ is $O(b_k^2 + r_k^2)$ (see Lemma~\ref{lem:E[xHx]}). Therefore, Algorithm~\ref{alg:reweight} is essentially searching for a weighting such that clean instances have overwhelming weights over dirty instances, and that the variance of the weighted instances is similar to that of the clean instances.
Recall that $\TC \subset T$ is the set of clean instances in $T$. Let $\tildeTC = \{ (x, y_x) \}_{x \in \TC}$ be the unrevealed labeled set where each instance is correctly annotated by $w^*$. The following proposition, which is similar to Lemma~4.7 of \citet{awasthi2017power} but with refinement, states that the reweighted hinge loss $\ell_{\tau_k}(w; p) := \sum_{x \in T} p(x) \ell_{\tau_k}(w; x, y_x)$, is a good proxy for the hinge loss evaluated exclusively on clean labeled instances $\tildeTC$.

\begin{proposition}\label{prop:l(TC)=l(p)}
Suppose  Assumption~\ref{as:x} and \ref{as:y} are satisfied, and $\eta \leq c_5 \epsilon$. For any phase $k$ of Algorithm~\ref{alg:main}, with probability $1- \frac{\delta_k}{4}$ over the draw of $\bar{T}$, we have $\sup_{w \in W_k} \envert[1]{\ell_{\tau_k}(w; \tildeTC) - \ell_{\tau_k}(w; p)} \leq \kappa$.
\end{proposition}

Note that though this proposition is phrased in terms of the hinge loss on pairs $(x, y_x)$, it is only used in the analysis and our algorithm does not require the knowledge of the labels $y_x$~--~the algorithm even does not need to exactly identify the set of clean instances $\TC$. As a result, the size of $\TC$ does not count towards our label complexity. Proposition~\ref{prop:outlier-guarantee} together with Proposition~\ref{prop:l(TC)=l(p)} implies that with high probability, Algorithm~\ref{alg:reweight} produces a desired probability distribution in polynomial time, which justifies its computational and statistical efficiency.

In addition, let $L_{\tau_k}(w) \defeq \EXP_{x \sim D_{w_{k-1}, b_{k}}}\sbr[1]{\ell_{\tau_k}\del{w; x, \sign{w^* \cdot x}}}$ be the expected loss on $D_{w_{k-1}, b_{k}}$. The following result links $L_{\tau_k}(w)$ to the empirical hinge loss on clean instances.

\begin{proposition}\label{prop:l(TC)=exp}
Under Assumption~\ref{as:x} and \ref{as:y}, and $\eta \leq c_5 \epsilon$, for any phase $k$ of Algorithm~\ref{alg:main}, with probability $1 - \frac{\delta_k}{4}$ over the draw of $\bar{T}$, we have $\sup_{w \in W_k} \envert[1]{L_{\tau_k}(w) - \ell_{\tau_k}(w; \tildeTC) } \leq \kappa$.
\end{proposition}

\subsection{Attribute and label-efficient empirical risk minimization}\label{subsec:ERM}

In light of Proposition~\ref{prop:l(TC)=l(p)}, one may want to find an iterate by minimizing its reweighted hinge loss $\ell_{\tau_k}(w; p)$. This requires collecting labels for all instances in $T$, which leads to a suboptimal label complexity $O\del[1]{s^2 \cdot \polylog{d, 1/\epsilon}}$. As a remedy, we perform a random sampling process, which draws $m_k$ instances from $T$ according to the distribution $p$ and then query their labels, resulting in the labeled instance set $S_k$. By standard uniform convergence arguments, it is expected that $\ell_{\tau_k}(w; S_k) \approx \ell_{\tau_k}(w; p)$ provided that $m_k$ is large enough, as is shown in the following proposition.
\begin{proposition}\label{prop:l(p)=l(S)}
Suppose that Assumption~\ref{as:x} and \ref{as:y} are satisfied. For any phase $k$ of Algorithm~\ref{alg:main}, with probability $1 - \frac{\delta_k}{4}$, we have $\sup_{w \in W_k} \abs{\ell_{\tau_k}(w; p) - \ell_{\tau_k}(w; S_k)} \leq \kappa$.
\end{proposition}

We remark that when establishing the performance guarantee, the $\ell_1$-norm constraint on the hypothesis space, together with an $\ell_{\infty}$-norm upper bound on the localized instance space, leads to a Rademacher complexity that has a linear dependence  on the sparsity (up to a logarithmic factor). Technically speaking, our analysis is more involved than that of \citet{awasthi2017power}: applying their analysis to the setting of learning sparse halfspaces along with the fact that the VC dimension of the class of $s$-sparse halfspaces is $O(s \log(d/s))$ would give a label complexity quadratic in $s$.



\subsection{Uniform concentration for unbounded data}\label{subsec:adamczak-bound}

Our analysis involves building uniform concentration bounds. The primary issue of applying standard concentration results, e.g. Theorem~1 of \citet{kakade2008complexity}, is that the instances are not contained in a pre-specified $\ell_\infty$-ball with probability $1$ under isotropic log-concave distribution. \citet{awasthi2017power,zhang2018efficient} construct a conditional distribution, on which the data are all bounded from above, and then measure the difference between this conditional distribution and the original one. We circumvent such technical complication by using the Adamczak's bound~\citep{adamczak2008tail} in the empirical process literature, which provides a generic way to analyze concentration inequalities for well-behaved distributions with unbounded support. See Appendix~\ref{sec:app:orlicz-concen} for a concrete treatment.

\subsection{Proof sketch of Theorem~\ref{thm:main}}

\begin{proof}
We first show that error rate of $v_k$ on $D_{w_{k-1}, b_k}$ is a constant, and that of $w_k$ follows since hard thresholding and $\ell_2$-norm projection can only deviate the error rate by a constant factor. Observe that in light of Proposition~\ref{prop:l(TC)=l(p)}, Proposition~\ref{prop:l(TC)=exp}, and Proposition~\ref{prop:l(p)=l(S)}, we have $\abs{ \ell_{\tau_k}(w; S_k) -  L_{\tau_k}(w)} \leq 3 \kappa$ for all $w \in W_k$. Therefore, if $w^* \in W_k$, by the optimality of $v_k$, we have $L_{\tau_k}(v_k) \leq \ell_{\tau_k}(v_k; S_k) + 3\kappa \leq \ell_{\tau_k}(w^*; S_k) + 4\kappa \leq L_{\tau_k}(w^*) + 7\kappa \leq 8\kappa$, where the last inequality is by Lemma~3.7 of \citet{awasthi2017power}. Since $L_{\tau_k}(v_k)$ always serves as an upper bound of $\err_{D_{w_{k-1}, b_{k}}}(v_k)$, the constant error rate on $D_{w_{k-1}, b_k}$ follows. Next we can use the analysis framework of margin-based active learning to show that such constant error rate ensures that the angle between $w_k$ and $w^*$ is as small as $O(2^{-k})$, which in turn implies $w^* \in W_{k+1}$. It remains to show $w^* \in W_1$; this can be easily seen by the definition of $W_1$: $W_1 = B_2(0, 1) \cap B_1(0, \sqrt{s})$. Hence, we conclude $w^* \in W_k$ for all $1\leq k \leq k_0$. Observe that the radius of $\ell_2$-ball of $W_{k_0}$ is as small as $\epsilon$, which, by a basic property of isotropic log-concave distributions, implies the error rate of $w_{k_0}$ on $D$ is less than $\epsilon$. 

The sample and label complexity bounds follow from our setting of $N_k$ and $m_k$, and the fact that $ b_k \in [\epsilon, \bar{c}/16]$ for all $k \leq k_0$. See Appendix~\ref{sec:app:progress} for the full proof.
\end{proof}

\section{Conclusion and Open Questions}\label{sec:conclusion}

We have presented a computationally efficient algorithm for learning sparse halfspaces under the challenging malicious noise model. Our algorithm leverages the well-established margin-based active learning framework, with a particular treatment on attribute efficiency, label complexity, and noise tolerance. We have shown that our theoretical guarantees for label complexity and noise tolerance are near-optimal, and the sample complexity of a passive learning variant of our algorithm is attribute-efficient, thanks to the set of new techniques proposed in this paper.

We raise three open questions for further investigation. First, as we discussed in Section~\ref{subsec:instance-reweight}, the sample complexity for concentration of $x\trans H x$ has a quadratic dependence on $s$. It would be interesting to study whether this is a fundamental limit of learning under isotropic log-concave distributions, or it can be improved by a more sophisticated localization scheme in the instance and the concept spaces. Second, while isotropic log-concave distributions bear favorable properties that fit perfectly in the margin-based framework, it would be interesting to examine whether the established results can be extended to heavy-tailed distributions. This may lead to a large error rate within the band that cannot be controlled at a constant level, and new techniques must be developed. Finally, it would be interesting to design computationally more efficient algorithms, e.g. stochastic gradient descent-type algorithms similar to~\cite{dasgupta2005analysis}, with comparable statistical guarantees.


\clearpage
\bibliographystyle{alpha}
\bibliography{ref}

\clearpage

\appendix

\section{Detailed Choices of Reserved Constants and Additional Notations}\label{sec:app:constants}

\paragraph{Constants.}
The absolute constants $c_0$, $c_1$ and $c_2$ are specified in Lemma~\ref{lem:logconcave}, and $c_3$ and $c_4$ are specified in Lemma~\ref{lem:err outside band}. $c_5$ and $c_6$ were clarified in Section~\ref{subsec:param-setting}. The definition of $c_7$, $c_8$, $c_9$ can be found in Lemma~\ref{lem:|x|_inf-D}, Lemma~\ref{lem:P(x in band)}, and Lemma~\ref{lem:|x|_inf-D_ub} respectively. The absolute constant $C_1$ acts as an upper bound of all $b_k$'s, and by our choice in Section~\ref{subsec:param-setting}, $C_1 = \bar{c}/16$. The absolute constant $C_2$ is defined in Lemma~\ref{lem:E[xHx]}. Other absolute constants, such as $C_3, C_4$ are not quite crucial to our analysis or algorithmic design. Therefore, we do not track their definitions. The subscript variants of $K$, e.g. $K_1$ and $K_2$, are also absolute constants but their values may change from appearance to appearance. We remark that the value of all these constants does not depend on the underlying distribution $D$ chosen by the adversary, but rather depends on the knowledge of $\calD$.

\paragraph{Pruning.}
Consider Algorithm~\ref{alg:main}. For each phase $k$, we sample a working set $\bar{T}$ and remove all instances that have large $\ell_{\infty}$-norm to obtain $T$ (Step~\ref{step:remove-ell-infty}), which is equivalent to intersecting it with the $\ell_\infty$-ball $B_{\infty}(0, \nu_k) := \cbr{x: \| x \|_\infty \leq \nu_k}$ where $\nu_k = c_9 \log\frac{48\abs{\bar{T}}d}{b_k\delta_k}$. This is motivated by Lemma~\ref{lem:|x|_inf-D_ub}, which states that with high probability, all clean instances in $\bar{T}$ are in $B_{\infty}(0, \nu_k)$.
Specifically, Denote by $\barTC$ (respectively $\barTD$) the set of clean (respectively dirty) instances in $\bar{T}$.
Lemma~\ref{lem:|x|_inf-D_ub} implies that with probability $1- \frac{\delta_k}{48}$, $\barTC \subset B_{\infty}(0, \nu_k)$.
Therefore, with high probability, all the instances in $\barTC$ are kept in this step and only instances in $\barTD$ may be removed.
Denote by $\TC = \barTC \cap B_{\infty}(0, \nu_k)$ and $\TD = \barTD \cap B_{\infty}(0, \nu_k)$; we therefore also have the decomposition
$T = \TC \cup \TD$. We finally denote by $\hatTC$ the unrevealed labeled set that corresponds to $\barTC$.

\begin{table}[h]
\centering
\caption{Summary of useful notations associated with the working set $\bar{T}$ at each phase $k$.}
\vspace{1em}
\begin{tabular}{ll}
\toprule
$\bar{T}$ & instance set obtained by calling $\oraclex$ conditioned on $\abs{w_{k-1} \cdot x} \leq b_k$\\
$\barTC$ & set of instances in $\bar{T}$ that $\oraclex$ draws from the distribution $D$\\
$\barTD$ & set of dirty instances in $\bar{T}$, i.e. $\bar{T} \backslash \barTC$\\
$T$ & set of instances in $\bar{T}$ that lie in $B_{\infty}(0, \nu_k)$\\
$\TC$ & set of instances in $\barTC$ that lie in $B_{\infty}(0, \nu_k)$\\
$\TD$ & set of instances in $\barTD$ that lie in $B_{\infty}(0, \nu_k)$\\
$\hatTC$ & unrevealed labeled set of $\barTC$\\
$\tildeTC$ & unrevealed labeled set of $\TC$\\
\bottomrule
\end{tabular}
\end{table}



\paragraph{Regularity condition on $D_{u, b}$.}
We will frequently work with the conditional distribution $D_{u, b}$ obtained by conditioning $D$ on the event that $x$ is in the band $\{x \in \Rd: \abs{u \cdot x} \leq b\}$. We give the following regularity condition to ease our terminology.

\begin{definition}
A conditional distribution $D_{u, b}$ is said to satisfy the regularity condition if one of the following holds: 1) the vector $u \in \Rd$ has unit $\ell_2$-norm and $0 < b \leq C_1$; 2) the vector $u$ is the zero vector and $b = C_1$.
\end{definition}
In particular, at each phase $k$ of Algorithm~\ref{alg:main}, $u$ is set to $w_{k-1}$ and $b$ is set to $b_k$. For $k=1$, $u = w_0$ is a zero vector, $b = b_1 = C_1$, satisfying the regularity condition. It is worth mentioning that at phase $1$ the conditional distribution $D_{u, b}$ boils down to $D$. For all $k \geq 2$, $u$ is a unit vector and $b \in (0, C_1]$ in view of our construction of $b_k$. Therefore, for all $k \geq 1$, $D_{w_{k-1}, b_k}$ satisfy the regularity condition.

\section{Useful Properties of Isotropic Log-Concave Distributions}\label{sec:logconcave-property}

We record some useful properties of isotropic log-concave distributions.

\begin{lemma}\label{lem:logconcave}
There are absolute constants $c_0, c_1, c_2 > 0$, such that the following holds for all isotropic log-concave distributions $D \in \calD$. Let $f_D$ be the density function. We have
\begin{enumerate}
\item \label{item:ilc:proj} Orthogonal projections of $D$ onto subspaces of $\Rd$ are isotropic log-concave;
\item \label{item:ilc:anti-anti-concen} If $d=1$, then $\Pr_{x \sim D}(a \leq x \leq b) \leq \abs{b-a}$;
\item \label{item:ilc:anti-concen} If $d=1$, then $f_D(x) \geq c_0$ for all $x \in [-1/9, 1/9]$;
\item \label{item:ilc:err=theta} For any two vectors $u, v \in \Rd$,
\begin{equation*}
c_1 \cdot \Pr_{x \sim D}\del{\sign{u\cdot x} \neq \sign{v \cdot x}} \leq \theta(u, v) \leq c_2 \cdot \Pr_{x \sim D}\del{\sign{u\cdot x} \neq \sign{v \cdot x}};
\end{equation*}
\item \label{item:ilc:tail} $\Pr_{x \sim D}\big( \twonorm{x} \geq t \sqrt{d} \big) \leq \exp(-t + 1)$.
\end{enumerate}
\end{lemma}

We remark that Parts~\ref{item:ilc:proj},~\ref{item:ilc:anti-anti-concen},~\ref{item:ilc:anti-concen}, and~\ref{item:ilc:tail} are due to \citet{lovasz2007geometry}, and Part~\ref{item:ilc:err=theta} is from \citet{vempala2010random,balcan2013active}. 

The following lemma is implied by the proof of Theorem~21 of \citet{balcan2013active}, which shows that if we choose a proper band width $b > 0$, the error outside the band will be small. This observation is crucial for controlling the error over the distribution $D$, and has been broadly recognized in the literature~\citep{awasthi2017power,zhang2018efficient}.

\begin{lemma}[Theorem~21 of \citet{balcan2013active}]\label{lem:err outside band}
There are absolute constants $c_3, c_4 > 0$ such that the following holds for all isotropic log-concave distributions $D \in \calD$. Let $u$ and $v$ be two unit vectors in $\Rd$ and assume that $\theta(u, v) = \alpha < \pi/2$. Then for any $b \geq \frac{4}{c_4} \alpha$, we have
\begin{equation*}
\Pr_{x \sim D}(\sign{u\cdot x} \neq \sign{v \cdot x}\ \text{and}\ \abs{v\cdot x} \geq b) \leq c_3 \alpha \exp\(- \frac{c_4 b}{2 \alpha}\).
\end{equation*}
\end{lemma}

\begin{lemma}[Lemma~20 of \citet{awasthi2016learning}]\label{lem:|x|_inf-D}
There is an absolute constant $c_7 > 0$ such that the following holds for all isotropic log-concave distributions $D \in \calD$. Draw $n$ i.i.d. instances from $D$ to form a set $S$. Then
\begin{equation*}
\Pr_{S \sim D^n}\( \max_{x \in S} \infnorm{x} \geq c_7 \log\frac{\abs{S}d}{\delta} \) \leq \delta.
\end{equation*}
\end{lemma}


\begin{lemma}\label{lem:E[wx^2]}
There is an absolute constant $\bar{C}_2 \geq 1$ such that the following holds for  all isotropic log-concave distributions $D \in \calD$ and all $D_{u, b}$ that satisfy the regularity condition:
\begin{equation*}
\sup_{w \in B_2(u, r)} \EXP_{x \sim D_{u, b}}\sbr[1]{(w\cdot x)^2} \leq \bar{C}_2 (b^2 + r^2).
\end{equation*}
\end{lemma}
\begin{proof}
When $u$ is a unit vector, Lemma~3.4 of~\cite{awasthi2017power} shows that there exists a constant $K_1$ such that
\begin{equation*}
\sup_{w \in B_2(u, r)} \EXP_{x \sim D_{u, b}}\sbr[1]{(w\cdot x)^2} \leq K_1 (b^2 + r^2).
\end{equation*}
When $u$ is a zero vector, $D_{u, b}$ reduces to $D$ and the constraint $w \in B_2(u, r)$ reads as $\twonorm{w} \leq r$. Thus we have
\begin{equation*}
\EXP_{x \sim D_{u, b}} \sbr[1]{ (w \cdot x)^2 } = \twonorm{w}^2 \leq r^2 < b^2 + r^2.
\end{equation*}
The proof is complete by choosing $\bar{C}_2 = K_1 + 1$.
\end{proof}

\begin{lemma}\label{lem:E[xHx]}
There is an absolute constant $C_2 \geq 2$ such that the following holds for  all isotropic log-concave distributions $D \in \calD$ and all $D_{u, b}$ that satisfy the regularity condition:
\begin{equation*}
\sup_{H \in \calM} \EXP_{x \sim D_{u, b}}\sbr[1]{ x\trans H x} \leq C_2 (b^2 + r^2),
\end{equation*}
where $\calM := \cbr{H \in \Rdd:\ H \succeq 0,\ \nuclearnorm{H} \leq r^2,\ \onenorm{H} \leq \rho^2}$.
\end{lemma}
%

\begin{proof}
Since $H \in \calM$ is a positive semidefinite matrix with trace norm at most $r^2$, it has eigendecomposition $H = \sum_{i=1}^{d} \lambda_i v_i v_i\trans$, where $\lambda_i \geq 0$ are the eigenvalues such that $\sum_{i=1}^{d} \lambda_i \leq r^2$, and $v_i$'s are orthonormal vectors in $\RR^d$. Thus,
\begin{equation*}
{x\trans H x} = \frac{1}{r^2} \sum_{i=1}^{d} {\lambda_i} { ( r v_i \cdot x)^2} \leq \frac{2}{r^2} \cdot \sum_{i=1}^{d} {\lambda_i} \sbr{ \del{(r v_i + u) \cdot x}^2 + (u \cdot x)^2 }.
\end{equation*}
Since $x$ is drawn from $D_{u, b}$, we have $(u \cdot x)^2 \leq b^2$. Moreover, applying Lemma~\ref{lem:E[wx^2]} with the setting of $w = r v + u$ implies that
\begin{equation*}
\sup_{v \in B_2(0, 1)} \EXP_{x \sim D_{u, b}} \sbr{ \del{(r v + u) \cdot x}^2 } \leq \bar{C}_2 (b^2 + r^2).
\end{equation*}
Therefore,
\begin{equation*}
\sup_{H \in \calM} \EXP_{x \sim D_{u, b}}\sbr[1]{ x\trans H x } \leq \frac{2}{r^2} \cdot \sum_{i=1}^{d} {\lambda_i} \del{ \bar{C}_2 (b^2 + r^2) + b^2 } \leq 2 (\bar{C}_2+1) (b^2 + r^2).
\end{equation*}
The proof is complete by choosing $C_2 = 2 ( \bar{C}_2 + 1)$.
\end{proof}

\begin{lemma}\label{lem:P(x in band)}
Let $c_8 =  \min\big\{2c_0, \frac{2c_0}{9 C_1}, \frac{1}{C_1}\big\}$. Then for all isotropic log-concave distributions $D \in \calD$ and all $D_{u, b}$ satisfying the regularity condition,
\begin{enumerate}

\item \label{item:prob-band:refined-lower}  $\Pr_{x \sim D}\( \abs{u \cdot x} \leq b \) \geq c_8 \cdot b$;

\item \label{item:prob-band:D-to-band} $\Pr_{x \sim D_{u, b}}(E ) \leq \frac{1}{c_8 b} \Pr_{x \sim D}(E)$ for any  event $E$.
\end{enumerate}
\end{lemma}
\begin{proof}
We first consider the case that $u$ is a unit vector.


For the lower bound, Part~\ref{item:ilc:anti-concen} of Lemma~\ref{lem:logconcave}  shows that the density function of the random variable $u\cdot x$ is lower bounded by $c_0$ when $\abs{u \cdot x} \leq 1/9$. Thus
\begin{align*}
\Pr_{x \sim D}\( \abs{u \cdot x} \leq b\) \geq \Pr_{x \sim D}\( \abs{u \cdot x} \leq \min\{b, {1}/{9}\}\) \geq 2 c_0 \min\{b, {1}/{9}\} \geq 2c_0 \min\bigg\{1, \frac{1}{9 C_1}\bigg\} \cdot b
\end{align*}
where in the last inequality we use the condition $b \leq C_1$.

For any event $E$, we always have
\begin{equation*}
\Pr_{x \sim D_{u, b}}(E) \leq \frac{\Pr_{x \sim D}(E)}{\Pr_{x \sim D}({\abs{u\cdot x}\leq b})} \leq \frac{1}{c_8 b} \Pr_{x \sim D}(E).
\end{equation*}

Now we consider the case that $u$ is the zero vector and $b = C_1$. Then $\Pr_{x \sim D}\( \abs{u \cdot x} \leq b \) = 1 \geq c_8 \cdot b$ in view of the choice $c_8$. Thus Part~\ref{item:prob-band:D-to-band} still follows. The proof is complete.
\end{proof}

\begin{lemma}\label{lem:|x|_inf-D_ub}
There exists an absolute constant $c_9 > 0$ such that the following holds for all isotropic log-concave distributions $D \in \calD$ and all $D_{u, b}$ that satisfy the regularity condition. Let $S$ be a set of i.i.d. instances drawn from $D_{u, b}$. Then
\begin{equation*}
\Pr_{S \sim D_{u, b}^n}\del{ \max_{x \in S} \infnorm{x} \geq c_9 \log\frac{\abs{S}d}{ b\delta} } \leq \delta.
\end{equation*}
\end{lemma}
\begin{proof}
Using Lemma~\ref{lem:|x|_inf-D} we have
\begin{equation*}
\Pr_{S \sim D^n}\del{ \max_{x \in S} \infnorm{x} \geq c_7 \log\frac{\abs{S}d}{\delta} } \leq \delta.
\end{equation*}
Thus, using Part~\ref{item:prob-band:D-to-band} of Lemma~\ref{lem:P(x in band)} gives
\begin{equation*}
\Pr_{S \sim D_{u, b}^n}\del{\max_{x \in S} \infnorm{x} \geq c_7 \log\frac{\abs{S}d}{\delta} } \leq \frac{\delta}{c_8b}.
\end{equation*}
The proof is complete by changing $\delta$ to $\delta' = \frac{\delta}{c_8 b}$.
\end{proof}


\section{Orlicz Norm and Concentration Results using Adamczak's Bound}
\label{sec:app:orlicz-concen}

The following notion of Orlicz norm~\citep{van2013bernstein,dudley2014uniform} is useful in handling random variables that have tails of the form $\exp(-t^\alpha)$ for general $\alpha$'s beyond $\alpha = 2$ (subgaussian) and $\alpha = 1$ (subexponential).

\begin{definition}[Orlicz norm]\label{def:orlicz-norm}
For any $z \in \R$, let $\psi_{\alpha}: z \mapsto \exp(z^\alpha) - 1$.
Furthermore, for a random variable $Z \in \R$ and $\alpha  > 0$, define $\orcnorm{Z}{\alpha}$, the Orlicz norm of $Z$ with respect to $\psi_{\alpha}$, as:
\begin{equation*}
\orcnorm{Z}{\alpha} = \inf\Big\{t > 0: \EXP_Z\sbr[1]{ \psi_\alpha \del{ {\abs{Z}}/{t} } }\leq 1\Big\}.
\end{equation*}
\end{definition}

We collect some basic facts about Orlicz norms in the following lemma; they can be found in Section 1.3 of \citet{van1996weak}.

\begin{lemma}\label{lem:orlicz-property}
Let $Z$, $Z_1$, $Z_2$ be real-valued random variables. Consider the Orlicz norm with respect to $\psi_{\alpha}$. We have the following:
\begin{enumerate}
\item $\orcnorm{\cdot}{\alpha}$ is a norm. For any $a \in \R$, $\orcnorm{a Z}{\alpha} = \abs{a} \cdot \orcnorm{Z}{\alpha}$;  $\orcnorm{Z_1 + Z_2}{\alpha} \leq \orcnorm{Z_1}{\alpha} + \orcnorm{Z_2}{\alpha}$.

\item \label{item:orlicz:regular} $\norm{Z}_p \leq \orcnorm{Z}{p} \leq p! \orcnorm{Z}{1}$ where $\norm{Z}_p \defeq \del{\EXP\sbr{\abs{Z}^p}}^{1/p}$.

\item \label{item:orlicz:hom} For any $p, \alpha > 0$, $\orcnorm{Z}{p}^\alpha = \orcnorm{Z^{\alpha}}{p/\alpha}$.

\item \label{item:orlicz:tail} If $\Pr\(\abs{Z} \geq t \) \leq K_1 \exp\del{-K_2 t^{\alpha}}$ for any $t \geq 0$, then $\orcnorm{Z}{\alpha} \leq \(\frac{2(\ln K_1 + 1)}{K_2}\)^{ 1/ \alpha}$.

\item If $\orcnorm{Z}{\alpha} \leq K$, then for all $t \geq 0$, $\Pr\( \abs{Z} \geq t\) \leq 2\exp\(-(\frac{t}{K})^\alpha\)$.

\end{enumerate}
\end{lemma}

The following auxiliary results, tailored to the localized sampling scheme in Algorithm~\ref{alg:main}, will also be useful in our analysis.
\begin{lemma}\label{lem:|x|-inf-orlicz-norm}
There exists an absolute constant $C_3 > 0$ such that the following holds for all isotropic log-concave distributions $D \in \calD$ and all $D_{u, b}$ that satisfy the regularity condition. Let $S = \{x_1, \dots, x_n\}$ be a set of $n$ instances drawn from $D_{u, b}$. Then
\begin{equation*}
\orcnorm{ \max_{x \in S} \infnorm{x}}{1} \leq C_3 \log \frac{nd}{b}.
\end{equation*}
Consequently,
\begin{equation*}
\EXP_{S \sim D_{u, b}^n } \sbr[2]{ \max_{x \in S} \infnorm{x} } \leq C_3 \log \frac{nd}{b}.
\end{equation*}
\end{lemma}
\begin{proof}
Let $Z$ be isotropic log-concave random variable in $\R$. Part~\ref{item:ilc:tail} of Lemma~\ref{lem:logconcave} shows that for all $t > 0$,
\begin{equation*}
\Pr( \abs{Z} > t ) \leq \exp(-t + 1).
\end{equation*}
Fix $i \in \{1, \dots, n\}$ and fix $j \in \{1, \dots, d\}$. Denote by $x_i^{(j)}$ the $j$-th coordinate of $x_i$. Part~\ref{item:ilc:proj} of Lemma~\ref{lem:logconcave} suggests that $x_i^{(j)}$ is isotropic log-concave. Thus, by Part~\ref{item:prob-band:D-to-band} of Lemma~\ref{lem:P(x in band)},
\begin{equation*}
\Pr_{x \sim D_{u, b}} \del[2]{\ \envert[1]{x_i^{(j)}} > t} \leq \frac{1}{c_8 b} \Pr_{x \sim D} \del{\ \envert[1]{x_i^{(j)}} > t} \leq \frac{1}{c_8 b} \exp(-t + 1).
\end{equation*}
Taking the union bound over $i \in \{1, \dots, n\}$ and $j \in \{1, \dots, d\}$, we have for all $t > 0$
\begin{equation*}
\Pr_{x \sim D_{u, b}}\del{\max_{x \in S} \infnorm{x} > t } \leq \frac{nd}{c_8 b} \exp(-t + 1).
\end{equation*}
Now Part~\ref{item:orlicz:tail} of Lemma~\ref{lem:orlicz-property} immediately implies that
\begin{equation*}
\orcnorm{ \max_{x \in S} \infnorm{x} }{1}  \leq C_3 \log \frac{nd}{b}
\end{equation*}
for some constant $C_3 > 0$. The second inequality of the lemma is an immediate result by combining the above and Part~\ref{item:orlicz:regular} of Lemma~\ref{lem:orlicz-property}.
\end{proof}

\subsection{Adamczak's bound}

In this section, we establish the key concentration results that will be used to analyze the performance of soft outlier removal and random sampling in Algorithm~\ref{alg:main}. Since we are considering the isotropic log-concave distribution, any unlabeled instance $x$ is unbounded. This prevents us from using standard concentration bounds, e.g. \cite{kakade2008complexity}. We henceforth appeal to the following generalization of Talagrand's inequality, due to \cite{adamczak2008tail}.

\begin{lemma}[Adamczak's bound]\label{lem:adamczak}
For any $\alpha \in (0,1]$, there exists a constant $\Lambda_\alpha > 0$, such that the following holds. Given any function class $\calF$, and a function $F$ such that for any $f \in \calF$, $\abs{f(x)} \leq F(x)$, we have
with probability at least $1-\delta$ over the draw of a set $S = \{x_1, \dots, x_n\}$ of i.i.d. instances from $D$,
\begin{align*}
\sup_{f \in \calF} \envert[3]{ \frac{1}{n} \sum_{i=1}^{n} {f(x_i)} - \EXP_{x \sim D} \sbr{f(x)} }
\leq \Lambda_\alpha \left( \EXP_{S \sim D^n} \sbr[3]{ \sup_{f \in \calF} \envert[3]{ \frac{1}{n} \sum_{i=1}^{n} {f(x_i)} - \EXP_{x \sim D} \sbr{f(x)} } }  \right. \\
 \left. +
\sqrt{ \frac{\sup_{f \in \calF} \EXP_{x \sim D} \sbr{(f(x))^2} \ln\frac1\delta}{n} }
 +
\frac{(\ln\frac1\delta)^{1/\alpha}}{n} \orcnorm{\max_{1\leq i \leq n} F(x_i) }{\alpha} \right).
\end{align*}
\end{lemma}

We first establish the following result that upper bounds the expected value of Rademacher complexity of linear classes by the Orlicz norm of the random instances.

\begin{lemma}\label{lem:expected-rad}
There exists an absolute constant $C_5 >0$ such that the following holds for all isotropic log-concave distributions $D \in \calD$ and all $D_{u, b}$ that satisfy the regularity condition. Let $S = \{x_1, \dots, x_n\}$ be a set of $n$ i.i.d. unlabeled instances drawn from $D_{u, b}$. Denote $W = B_2(u, r) \cap  B_1(u, \rho)$. Let a sequence of random variables $Z = \{z_1, \dots, z_n\}$ be drawn from a distribution supported on a bounded interval $[-\lambda, \lambda]$ for some $\lambda > 0$. Let $\sigma = \{\sigma_1, \dots, \sigma_n\}$, where the $\sigma_i$'s are i.i.d. Rademacher random variables independent of $S$ and $Z$. We have:
\begin{equation*}
\EXP_{S, Z, \sigma}\sbr[3]{ \sup_{w \in W} \envert[3]{ \sum_{i=1}^{n} \sigma_i z_i (w \cdot x_i)} } \leq \lambda b \sqrt{n} +  C_5 \rho \lambda \sqrt{n \log d} \cdot \log\frac{nd}{b}.
\end{equation*}
\end{lemma}
\begin{proof}
Let $V = B_2(0, r) \cap B_1(0, \rho)$ so that any $w \in W$ can be expressed as $w = u + v$ for some $v \in V$. First, conditioned on $S$ and $Z$, we have that
\begin{equation*}
\EXP_{\sigma}\sbr[3]{ \sup_{v \in V} \envert[3]{ \sum_{i=1}^{n} \sigma_i z_i (v\cdot x_i)} } \leq \rho \sqrt{2n \log(2d)} \cdot \max_{1 \leq i \leq n} \infnorm{z_i x_i} \leq \rho \lambda \sqrt{2n \log(2d)} \cdot \max_{1 \leq i \leq n} \infnorm{x_i}.
\end{equation*}
Thus,
\begin{align}
\EXP_{S, Z, \sigma}\sbr[3]{ \sup_{v \in V} \envert[3]{ \sum_{i=1}^{n} \sigma_i z_i (v\cdot x_i)} } &\leq \rho \lambda \sqrt{2n \log(2d)} \cdot   \EXP_{S} \sbr{ \max_{1 \leq i \leq n} \infnorm{ x_i} } \notag\\
&\leq C_5 \rho \lambda \sqrt{n \log d} \cdot \log\frac{nd}{b},\label{eq:tmp:exprad}
\end{align}
where the second inequality follows from Lemma~\ref{lem:|x|-inf-orlicz-norm}.

On the other side, using the fact that for any random variable $A$, $\EXP[A] \leq \del{\EXP[A^2]}^{1/2}$, we have
\begin{align*}
\EXP_{S, Z, \sigma} \sbr[3]{\ \envert[3]{\sum_{i=1}^{n} \sigma_i z_i (u \cdot x_i)} } &\leq \sqrt{ \EXP_{S, Z, \sigma} \sbr[4]{ \del[3]{\sum_{i=1}^{n} \sigma_i z_i (u \cdot x_i)}^2 } } \\
&= \sqrt{ \EXP_{S, Z} \sbr[4]{\sum_{i=1}^{n} z_i^2 (u\cdot x_i)^2} } \leq \sqrt{n b^2 \lambda^2},
\end{align*}
where in the equality we use the observation that $\EXP_{S,  Z, \sigma}\sbr{ \sigma_i \sigma_j z_i z_j (u\cdot x_i) (u \cdot x_j)} = 0$ when $i \neq j$, and in the last inequality we used the condition that $x_i$ is drawn from $D_{u, b}$. Combining the above with \eqref{eq:tmp:exprad} we obtain the desired result.
\end{proof}

\subsection{Uniform concentration of hinge loss}

\begin{proposition}\label{prop:concentrate-hinge}
There exists an absolute constant $C_6 >0$ such that the following holds for all isotropic log-concave distributions $D \in \calD$ and all $D_{u, b}$ that satisfy the regularity condition. Let $S = \{x_1, \dots, x_n\}$ be a set of $n$ i.i.d. unlabeled instances drawn from $D_{u, b}$ which satisfies the regularity condition. Let $y_x = \sign{w^* \cdot x}$ for any $x \sim D_{u, b}$. Denote $W = B_2(u, r) \cap B_1(u, \rho)$ and let $G(w) =\frac{1}{n} \sum_{i=1}^n \ell_\tau(w; x_i, y_{x_i}) - \EXP_{x \sim D_{u,b}} \big[\ell_\tau\(w; x, y_x\) \big]$.  Then with probability $1-\delta$,
\begin{align*}
\sup_{w \in W}\abs{G(w)} \leq C_6 \del{ \frac{b + \rho \sqrt{\log d} \log\frac{nd}{b} }{\tau \sqrt{n}} + \frac{b + r}{\tau \sqrt{n}} \sqrt{\log\frac{1}{\delta}} + \frac{b + \rho\log\frac{nd}{b}}{\tau n} \log\frac{1}{\delta} }.
\end{align*}
In particular, suppose $b = O(r)$, $\rho = O(\sqrt{s} r)$ and $\tau = \Omega(r)$. Then we have: for any $t > 0$, a sample size $n = \tilde{O}\Big(\frac{1}{t^2} s \log^2\frac{d}{b} \cdot \log\frac{d}{\delta} \Big)$ suffices to guarantee that with probability $1-\delta$, $\sup_{w \in W} \abs{G(w)} \leq t$.
\end{proposition}

\begin{proof}
We will use Lemma~\ref{lem:adamczak} with function class $\calF = \cbr{ (x,y) \mapsto \ell_\tau(w; x, y): w \in W }$ and the Orlicz norm with respect to $\psi_1$. We define $F(x, y) = 1 + \frac{b}{\tau} + \frac{\rho}{\tau} \| x \|_\infty$. It can be seen that for every $w \in W$,
\begin{equation*}
\abs{\ell_\tau(w; x, y)} \leq 1 + \frac{\abs{w \cdot x}}{\tau} \leq 1 + \frac{{u}\cdot{x}}{\tau} + \frac{{(w-u)}\cdot{x}}{\tau} \leq 1 + \frac{b}{\tau} + \frac{\rho}{\tau} \| x \|_\infty = F(x,y).
\end{equation*}
That is, for every $f$ in $\calF$, $\abs{f(x,y)} \leq F(x,y)$.

\vspace{0.1in}
\noindent{\bfseries Step 1.}
We upper bound $\orcnorm{ \max_{1\leq i \leq n} F(x_i, y_{x_i}) }{1}$. Since $\orcnorm{\cdot}{1}$ is a norm, we have
\begin{align}
\orcnorm{ \max_{1\leq i \leq n} F(x_i, y_{x_i}) }{1} &\leq \orcnorm{ 1 + \frac{b}{\tau} }{1} +\orcnorm{  \frac{\rho}{\tau} \cdot  \max_{1\leq i \leq n} \infnorm{x_i} }{1}  \notag\\
&= 1 + \frac{b}{\tau} + \frac{\rho}{\tau} \cdot \orcnorm{ \max_{1\leq i \leq n} \infnorm{x_i} }{1} \notag\\
&\leq 1 + \frac{b}{\tau} + \frac{C_3 \rho}{\tau} \log \frac{nd}{b},\label{eq:tmp:bound-F(x,y)}
 \end{align}
where we applied Lemma~\ref{lem:|x|-inf-orlicz-norm} in the last inequality.

\vspace{0.1in}
\noindent{\bfseries Step 2.}
Next, we upper bound $\sup_{w \in W} \EXP_{x \sim D_{u, b}}\sbr{ (\ell_\tau(w; x ,y_x ))^2 }$. For all $w$ in $W$, we have
\begin{equation}\label{eq:tmp:bound-sup-f(x,y)}
\sup_{w \in W} \EXP_{x \sim D_{u, b}} \sbr{ (\ell_\tau(w; x, y_x))^2 }  \leq 2 \cdot \sup_{w \in W} \EXP_{x \sim D_{u, b}} \sbr[3]{1 + \frac{(w\cdot x)^2}{\tau^2}} \leq 2 + 2 \bar{C}_2 \cdot \frac{r^2 + b^2}{\tau^2}
\end{equation}
where the last inequality uses Lemma~\ref{lem:E[wx^2]}.

\vspace{0.1in}
\noindent{\bfseries Step 3.}
Finally, we upper bound $\EXP_{S \sim D_{u, b}^n} \sbr{\sup_{w \in W} \abs{G(w)} }$.
Let $\sigma = \{\sigma_1, \dots, \sigma_n\}$ where each $\sigma_i$ is an i.i.d. draw from the Rademacher distribution. We have
\begin{align}
 \EXP_{S} \sbr[3]{ \sup_{w \in W} \abs{G(w)} } &\leq \frac{2}{n} \EXP_{S, \sigma} \sbr[3]{ \sup_{w \in W} \envert[3]{ \sum_{i=1}^{n} \sigma_i \ell_{\tau}\del{w; x_i, y_{x_i}} } }\notag\\
&\leq \frac{2}{\tau n} \EXP_{S, \sigma} \sbr[3]{ \sup_{w \in W} \envert[3]{ \sum_{i=1}^{n} \sigma_i y_{x_i} (w \cdot x_i) } } \notag\\
&\leq \frac{2b}{\tau \sqrt{n}}+ \frac{2C_5\rho}{\tau} \cdot \sqrt{\frac{\log d}{n}} \cdot  \log\frac{nd}{b}.\label{eq:tmp:bound-EXP}
\end{align}
In the above, the first inequality used standard symmetrization arguments; see, for example, Lemma~26.2 of \citet{shalev2014understanding}. In the second inequality, we used the contraction property of Rademacher complexity and the fact that $\ell_{\tau}(w; x, y)$ can be seen as a $\frac{1}{\tau}$-Lipschitz function $\phi(a) = \max\big\{0, 1 - \frac{a}{\tau} \big\}$ applied on input $a = y w \cdot x$. In the last inequality, we applied Lemma~\ref{lem:expected-rad} with the fact that $\abs{y_{x_i}} \leq 1$.

\vspace{0.1in}
\noindent{\bfseries Putting together.}
The first inequality of the proposition follows from combining \eqref{eq:tmp:bound-F(x,y)}, \eqref{eq:tmp:bound-sup-f(x,y)}, and \eqref{eq:tmp:bound-EXP}, and using Lemma~\ref{lem:adamczak} with $\calF$ and $\psi_1$. Under our choice of $(b, r, \rho, \tau)$, with some calculation we obtain the bound of $n$.
\end{proof}

\subsection{Uniform concentration of relaxed sparse PCA}

%

\begin{proposition}\label{prop:concentrate-spca}
There exists an absolute constant $C_7 >0$ such that the following holds for all isotropic log-concave distributions $D \in \calD$ and all $D_{u, b}$ that satisfy the regularity condition. Let $S = \{x_1, \dots, x_n\}$ be a set of $n$ i.i.d. unlabeled instances drawn from $D_{u, b}$. Denote $G(H) = \frac{1}{n} \sum_{i=1}^n { x_i\trans H x_i } - \EXP_{x \sim D_{u, b}}\sbr[1]{ x\trans H x}$. Then with probability $1-\delta$,
\begin{align*}
\sup_{H \in \calM} \abs{G(H)} \leq C_7 \rho^2 \log^2\frac{nd}{b} \del[3]{ \sqrt{\frac{\log d}{n}} + \sqrt{\frac{\log({1}/{\delta})}{n}} + \frac{\log^2\frac{1}{\delta}}{n} }.
\end{align*}
In particular, suppose $\rho = O(\sqrt{s} r)$ and $r = O(b)$. Then we have: for any $t > 0$, a sample size
\begin{equation*}
n = \tilde{O}\del{ \frac{1}{t^2} s^2 b^4 \log^4\frac{d}{b} \cdot \del{ \log d + \log^2\frac{1}{\delta} } }
\end{equation*}
suffices to guarantee that with probability $1-\delta$, $\sup_{H\in \calM} \abs{G(H)} \leq t$.
\end{proposition}
\begin{proof}
Recall that $\calM = \cbr{ H \in \Rdd: H \succeq 0, \spenorm{H} \leq r^2, \onenorm{H} \leq \rho^2 }$. For any matrix $H$, we denote by $H_{ij}$ the $(i, j)$-th entry of the matrix $H$. For any vector $x$, we denote by $x^{(i)}$ the $i$-th coordinate of $x$.

We will use Lemma~\ref{lem:adamczak} with function class $\calF = \cbr{x \mapsto x\trans H x: H \in \calM}$ and the Orlicz norm with respect to $\psi_{0.5}$. Consider the function $f(x) :=  x \trans H x$ parameterized by $H \in \calM$. First, we wish to find a function $F(x)$ that upper bounds $\abs{f(x)}$. It is easy to see that
\begin{equation}\label{eq:tmp:bound-xHx}
\envert[2]{x\trans H x} = \envert[2]{ \sum_{i, j} H_{ij} x^{(i)} x^{(j)} } \leq \infnorm{x}^2 \sum_{i, j} \abs{H_{ij}} \leq \rho^2 \infnorm{x}^2.
\end{equation}
Thus it suffices to choose $F(x) = \rho^2 \infnorm{x}^2$.

\vspace{0.1in}
\noindent{\bfseries Step 1.}
We first bound $\orcnorm{\sqrt{\max_{1 \leq i \leq n}F(x_i)}}{1} = \orcnorm{ \rho \cdot \max_{1 \leq i \leq n} \infnorm{x_i}}{1} \leq C_3 \rho \log\frac{nd}{b}$ by Lemma~\ref{lem:|x|-inf-orlicz-norm}.
By Part~\ref{item:orlicz:hom} of Lemma~\ref{lem:orlicz-property}, $\orcnorm{\max_{1 \leq i \leq n} F(x) }{0.5}$ equals $\orcnorm{\sqrt{\max_{1 \leq i \leq n} F(x)}}{1}^2$. Thus
\begin{equation}\label{eq:tmp:wx:bound-1}
\orcnorm{ \max_{1 \leq i \leq n} F(x) }{0.5} \leq  \del{ C_3 \rho \log \frac{nd}{b} }^2.
\end{equation}

\vspace{0.1in}
\noindent{\bfseries Step 2.}
Next we upper bound $\sup_{f \in \calF} \EXP_{x \sim D_{u, b}} \sbr{ (f(x))^2 }$ where we remark that taking the superum over $f \in \calF$ is equivalent to taking that over $H \in \calM$. Since $\abs{f(x)} \leq F(x)$, we have
\begin{equation*}
(f(x))^2 \leq (F(x))^2 \leq  \rho^4 \infnorm{x}^4.
\end{equation*}
In view of Part~\ref{item:orlicz:regular} of Lemma~\ref{lem:orlicz-property}, we have
\begin{equation}
\del{ \EXP_{x \sim D_{u, b}} \sbr{ \infnorm{x}^4 } }^{1/4} \leq 24 \orcnorm{\infnorm{x}}{1} \leq 24 C_3 \log\frac{d}{b},
\end{equation}
where the last inequality follows from Lemma~\ref{lem:|x|-inf-orlicz-norm}. Hence,
\begin{equation}\label{eq:tmp:wx-bound-2}
\sup_{f \in \calF} \EXP_{x \sim D_{u, b}} \sbr{ (f(x))^2 } \leq K_1  \rho^4 \log^4 \frac{d}{b}
\end{equation}
for some absolute constant $K_1 > 0$.

\vspace{0.1in}
\noindent{\bfseries Step 3.}
Finally, we upper bound $\EXP_{S \sim D^n} \sbr{ \sup_{f \in \calF} \abs{ \frac{1}{n}\sum_{i=1}^{n} f(x_i) - \EXP_{x \sim D_{u,b}} \sbr{ f(x)} } }$. Let $\sigma = \{\sigma_1, \dots, \sigma_n\}$ where $\sigma_i$'s are independent draw from the Rademacher distribution. By standard symmetrization arguments (see e.g. Lemma~26.2 of \citet{shalev2014understanding}), we have
\begin{equation}\label{eq:tmp:var-rad}
 \EXP_{S} \sbr[3]{ \sup_{f \in \calF} \abs{ G(v, H) } }
\leq \frac{2}{n} \EXP_{S, \sigma} \sbr[3]{ \sup_{f \in \calF} \envert[3]{ \sum_{i=1}^n \sigma_i f(x_i)} } = \frac{2}{n}  \EXP_{S, \sigma} \sbr[3]{ \sup_{H \in \calM} \envert[3]{ \sum_{i=1}^n \sigma_i x_i\trans H x_i} }.
\end{equation}
We first condition on $S$ and consider the expectation over $\sigma$. For a matrix $H$, we use $\ve(H)$ to denote the vector obtained by concatenating all of the columns of $H$; likewise for $x_i x_i\trans$. It is crucial to observe that with this notation, for any $H \in \calM$, we have $\onenorm{\ve(H)} = \onenorm{H} \leq \rho^2$. It follows that
\begin{align*}
\EXP_\sigma \sbr[4]{\ \envert[3]{ \sup_{H \in \calM} \sum_{i=1}^n \sigma_i x_i\trans H x_i } }
& \leq
\EE_\sigma \sbr[4]{ \sup_{ H: \onenorm{\ve(H)} \leq \rho^2 } \envert[3]{ \sum_{i=1}^n \sigma_i \inner{ \ve(H) }{ \ve(x_i x_i\trans) }} }\\
& \leq
\rho^2 \sqrt{n \ln(2d^2)} \cdot \max_{1 \leq i \leq n} \infnorm{ \ve(x_i x_i\trans)} \cdot  \\
& =
\rho^2 \sqrt{n \ln(2d^2)} \cdot  \max_{1\leq i \leq n} \infnorm{x_i}^2.
\end{align*}
where the second inequality is from Lemma~\ref{lem:L1-rad}, and the equality is from the observation that $\| \ve(x_i x_i^\top) \|_\infty = \| x_i \|_\infty^2$. Therefore,
\begin{align*}
\EXP_{S,  \sigma} \sbr[4]{\ \envert[3]{ \sup_{H \in \calM} \sum_{i=1}^n \sigma_i x_i^\top H x_i } }
&\leq \rho^2 \sqrt{n \ln(2d^2)} \cdot  \EXP_S \sbr{ \max_{1 \leq i \leq n} \| x_i \|_\infty^2} \\
&\leq \rho^2 \sqrt{2 n \ln(2d)} \cdot 2 \orcnorm{\max_{1 \leq i \leq n} \infnorm{x_i}}{1}^2 \\
&\leq \rho^2 \sqrt{2 n \ln(2d)} \cdot C_3^2 \log^2\frac{nd}{b},
\end{align*}
where the second inequality follows from Part~\ref{item:orlicz:regular} of Lemma~\ref{lem:orlicz-property}, and the last inequality follows from Lemma~\ref{lem:|x|-inf-orlicz-norm}. In summary,
\begin{equation}\label{eq:tmp:var-3-3}
\EXP_{S,  \sigma} \sbr[4]{ \sup_{f \in \calF} \envert[3]{ \sum_{i=1}^n \sigma_i x_i\trans H x_i} } \leq K_2 \sqrt{n \ln d} \cdot \rho^2 \log^2\frac{nd}{b}
\end{equation}
for some constant $K_2 > 0$.


Combining \eqref{eq:tmp:var-rad} and \eqref{eq:tmp:var-3-3}, we have
\begin{equation}\label{eq:tmp:wx-bound-3}
 \EXP_{S} \sbr[3]{ \sup_{f \in \calF} \abs{ G(H) } } \leq \frac{K_3 \sqrt{\log d}}{\sqrt{n}} \cdot {\rho^2}\log^2\frac{nd}{b}.
\end{equation}

\vspace{0.1in}
\noindent{\bfseries Putting together.}
Combining \eqref{eq:tmp:wx:bound-1}, \eqref{eq:tmp:wx-bound-2}, \eqref{eq:tmp:wx-bound-3}, and using Lemma~\ref{lem:adamczak} gives the first inequality of the proposition. Under our setting of $(b, r, \rho)$, by some calculation we obtain the bound of $n$. The proof is complete.
\end{proof}

\section{Performance Guarantee of Algorithm~\ref{alg:main}}\label{sec:app:guarantee}

In this section, we leverage all the tools from previous sections to establish the performance guarantee of Algorithm~\ref{alg:main}. Our main theorem, Theorem~\ref{thm:main}, follows from the analysis of each step of the algorithm, as we describe below.

\subsection{Analysis of sample complexity}\label{sec:app:N}

Recall that we refer to the number of calls to $\oraclex$ as the sample complexity of Algorithm~\ref{alg:main}. In order to obtain $n_k$ instances residing the band $X_k := \{x: \abs{w_{k-1} \cdot x} \leq b_k\}$, we have to call $\oraclex$ sufficient times.

\begin{lemma}[Restatement of Lemma~\ref{lem:N}]\label{lem:N-restate}
Consider phase $k$ of Algorithm~\ref{alg:main} for any $k \geq 1$. Suppose that Assumption~\ref{as:x} and \ref{as:y} are satisfied. Further assume $\eta < \frac{1}{2}$. By making a number of $N_k = O\del[2]{\frac{1}{ b_k}\del[1]{ n_k + \log\frac{1}{\delta_k} }}$ calls to the instance generation oracle $\oraclex$, we will obtain $n_k$ instances that fall into $X_k$ with probability $1- \frac{\delta_k}{4}$.
\end{lemma}
\begin{proof}
By Lemma~\ref{lem:P(x in band)}
\begin{equation*}
\Pr_{x \sim D}(x \in X_k) \geq c_8 b_k.
\end{equation*}
This implies that
\begin{align*}
&\ \Pr_{x \sim \oraclex}(x \in X_k \ \text{and}\  x \text{ is clean}) \\
=&\ \Pr_{x \sim \oraclex}(x \in X_k \mid  x \text{ is clean}) \cdot \Pr_{x \sim \oraclex}( x \text{ is clean} )\\
  \geq&\ c_8 b_k (1-\eta).
\end{align*}

We want to ensure that by drawing $N_k$ instances from $\oraclex$, with probability at least $1- \frac{\delta_k}{4}$, $n_k$ out of them fall into the band $X_k$. We apply the second inequality of Lemma~\ref{lem:chernoff} by letting $Z_i=\ind{x_i \in X_k \ \text{and}\ x_i \text{ is clean}}$ and $\alpha = 1/2$, and obtain
\begin{equation*}
\Pr\del{ \abs{\bar{\TC}} \leq \frac{c_8 b_k (1-\eta)}{2} N_k } \leq \exp\del{-\frac{c_8 b_k (1-\eta) N_k}{8}},
\end{equation*}
where the probability is taken over the event that we make a number of $N_k$ calls to $\oraclex$. Thus, when $N_k \geq \frac{8}{c_8 b_k (1- \eta)}\del{ n_k + \ln\frac{4}{\delta_k} }$, we are guaranteed that at least $n_k$ samples from $\oraclex$ fall into the band $X_k$ with probability $1 - \frac{\delta_k}{4}$. The lemma follows by observing $\eta < \frac{1}{2}$.
\end{proof}

\subsection{Analysis of pruning and the structure of $\bar{T}$}\label{sec:app:prune}

With the instance set $\bar{T}$ on hand, we estimate the empirical noise rate after applying pruning (Step~\ref{step:remove-ell-infty}) in Algorithm~\ref{alg:main}. Recall that $n_k = \abs{\bar{T}}$, i.e. the number of unlabeled instances before pruning.

\begin{lemma}\label{lem:noise-rate-in-band}
Suppose that Assumption~\ref{as:x} and Assumption~\ref{as:y} are satisfied. Further assume $\eta < \frac{1}{2}$. If $D_{u, b}$ satisfies the regularity condition, we have
\begin{equation*}
\Pr_{x \sim \oraclex}\del{ x\ \text{is\ dirty} \mid x \in X_{u, b} } \leq \frac{2\eta}{c_8 b}
\end{equation*}
where $c_8$ was defined in Lemma~\ref{lem:P(x in band)} and $X_{u, b} := \cbr{x \in \Rd: \abs{u \cdot x} \leq b}$.
\end{lemma}
\begin{proof}
For an instance $x$, we use $\mathrm{tag}_x = 1$ to denote that $x$ is drawn from $D$, and use $\mathrm{tag}_x = -1$ to denote that $x$ is adversarially generated.

We first calculate the probability that an instance returned by $\oraclex$ falls into the band $X_{u, b}$ as follows:
\begin{align*}
&\ \Pr_{x \sim \oraclex} \del{x \in X_{u, b}} \\
=&\ \Pr_{x \sim \oraclex}\del{x \in X_{u, b} \ \text{and}\ \mathrm{tag}_x = 1} + \Pr_{x \sim \oraclex}\del{x \in X_{u, b} \ \text{and} \ \mathrm{tag}_x = -1}\\
\geq&\ \Pr_{x \sim \oraclex}\del{x \in X_{u, b} \ \text{and}\ \mathrm{tag}_x = 1}\\
=&\ \Pr_{x \sim \oraclex}\del{x \in X_{u, b} \mid \mathrm{tag}_x = 1} \cdot \Pr_{x \sim \oraclex}\del{\mathrm{tag}_x = 1}\\
=&\ \Pr_{x \sim D}\del{x \in X_{u, b}} \cdot \Pr_{x \sim \oraclex}\del{\mathrm{tag}_x = 1}\\
\stackrel{\zeta}{\geq}&\ c_8 b \cdot (1- \eta)\\
\geq&\ \frac{1}{2}c_8 b,
\end{align*}
where in the inequality $\zeta$ we applied Part~\ref{item:prob-band:refined-lower} of Lemma~\ref{lem:P(x in band)}. It is thus easy to see that
\begin{equation*}
\Pr_{x \sim \oraclex}\del{\textrm{tag}_x = -1 \mid x \in X_{u, b}} \leq \frac{ \Pr_{x \sim \oraclex}\del{ \textrm{tag}_x = -1} }{ \Pr_{x \sim \oraclex} \del{x \in X_{u, b}} } \leq \frac{2\eta}{c_8 b},
\end{equation*}
which is the desired result.
\end{proof}

\begin{lemma}\label{lem:pruning}
Suppose that Assumptions~\ref{as:x} and \ref{as:y} are satisfied. Further assume $\eta \leq c_5 \epsilon$. For any $1 \leq k \leq k_0$, if $n_k \geq \frac{6}{\xi_k}\ln\frac{48}{\delta_k}$, then with probability $1 - \frac{\delta_k}{24}$ over the draw of $\bar{T}$, the following results hold simultaneously:
\begin{enumerate}
\item $\TC = \barTC$ and hence $\tildeTC = \hatTC$, i.e. all clean instances in $\bar{T}$ are intact after pruning;
\item  $\frac{\abs{\TD}}{\abs{T}} \leq \xi_k$, i.e. the empirical noise rate after pruning is upper bounded by $\xi_k$;
\item $\abs{\TC} \geq (1-\xi_k) n_k$.
\end{enumerate}
In particular, with the hyper-parameter setting in Section~\ref{subsec:param-setting}, $\abs{\TC} \geq \frac{1}{2} n_k$.
\end{lemma}
\begin{proof}
Let us write events $E_1 := \cbr{ \TC = \barTC}$, $E_2 :=\cbr{ \abs{\barTD} \leq \xi_k n_k }$. We bound the probability of the two events over the draw of $\bar{T}$.

Recall that Lemma~\ref{lem:|x|_inf-D_ub} implies that with probability $1- \frac{\delta_k}{48}$, all instances in $\barTC$ are in the $\ell_\infty$-ball $B_{\infty}(0, \nu_k)$ for $\nu_k = c_9 \log\frac{48\abs{\bar{T}}d}{b_k \delta_k}$, which implies $\Pr(E_1 ) \geq 1- \frac{\delta_k}{48}$.

We next calculate the noise rate within the band $X_k \defeq \{x: \abs{w_{k-1} \cdot x} \leq b_k\}$ by Lemma~\ref{lem:noise-rate-in-band}:
\begin{equation*}
\Pr_{x \sim \oraclex}( x\ \text{is\ dirty} \mid x \in X_k) \leq \frac{2\eta}{c_8 b_k} = \frac{2\eta}{c_8  \bar{c} \cdot 2^{-k-3}} \leq \frac{\pi}{ c_8 \bar{c} c_1} \cdot \frac{\eta}{\epsilon} \leq \frac{\pi c_5}{ c_8 \bar{c} c_1}  \leq \frac{\xi_k}{2},
\end{equation*}
where the equality applies our setting on $b_k$, the second inequality uses the condition $k \leq k_0$ and the setting $k_0 = \log\big(\frac{\pi}{16 c_1 \epsilon}\big)$, and the last inequality is guaranteed  by our choice of $c_5$. Now we apply the first inequality of Lemma~\ref{lem:chernoff} by specifying $Z_i = \ind{x_i\ \text{is\ dirty}}$, $\alpha = 1$ therein, which gives
\begin{equation*}
\Pr\del{ \abs{\barTD} \geq \xi_k n_k } \leq \exp\del{ - \frac{\xi_k n_k}{6}},
\end{equation*}
where the probability is taken over the draw of $\bar{T}$. This implies $\Pr(E_2) \geq 1 - \frac{\delta_k}{48}$ provided that $n_k \geq \frac{6}{\xi_k}\ln\frac{48}{\delta_k}$. 

By union bound, we have $\Pr(E_1 \cap E_2) \geq 1 - \frac{\delta_k}{24}$. We show that on the event $E_1 \cap E_2$, the second and third parts of the lemma follow. To see this, we note that it trivially holds that $\frac{\abs{\TD}}{\abs{T}} \leq \frac{\abs{\barTD}}{n_k}$ since only dirty instances have chance to be removed. This proves the second part. Also, it is easy to see that $\abs{\TC} = \abs{\barTC} = \abs{\bar{T}} - \abs{\barTD} \geq (1 - \xi_k) \abs{\bar{T}}$, which is exactly the third part. 
\end{proof}

\subsection{Analysis of Algorithm~\ref{alg:reweight}}\label{sec:app:instance-reweight}

\begin{lemma}[Restatement of Lemma~\ref{lem:xHx-emp}]\label{lem:xHx-emp-restate}
Suppose that Assumption~\ref{as:x} and \ref{as:y} are satisfied, and that $\eta \leq c_5 \epsilon$. There exists a constant $C_2 > 2$ such that the following holds. Consider phase $k$ of Algorithm~\ref{alg:main} for any $1 \leq k \leq k_0$.  Denote by $\calM_k$ the constraint set of \eqref{eq:spca-relax}. If $\abs{\TC} = \tilde{O}\del[2]{s^2 \log^4\frac{d}{b_k} \cdot \del[1]{\log d + \log^2 \frac{1}{\delta_k}}}$, then with probability $1 - \frac{\delta_k}{24}$ over the draw of $\TC$, we have
\begin{enumerate}
\item $\sup_{H \in \calM_k}\frac{1}{\abs{\TC}} \sum_{x \in \TC} x\trans H x \leq 2 C_2 (b_k^2 + r_k^2)$;

\item \label{item:outlier:avg-var-TC} $\sup_{w \in W_k} \frac{1}{\abs{\TC}} \sum_{x \in \TC} (w\cdot x)^2 \leq 5 C_2 \del{b_k^2 + r_k^2}$.
\end{enumerate}

\end{lemma}
\begin{proof}
The first part is an immediate result by combining Proposition~\ref{prop:concentrate-spca} and Lemma~\ref{lem:E[xHx]}, and recognizing our setting of $b_k$ and $r_k$.

To see the second part, for any $w \in W_k$, we can upper bound $(w \cdot x)^2$ as follows:
\begin{equation*}
(w\cdot x)^2 \leq 2 (w_{k-1} \cdot x)^2 + 2 (v \cdot x)^2 \leq 2 b_k^2 + 2 x\trans (v v\trans ) x,
\end{equation*}
where $v = w - w_{k-1} \in B_2(0, r_k) \cap B_1(0, \rho_k)$. Hence it is easy to see that $v v\trans$ lies in $\calM_k$. This indicates that for any $w \in W_k$, there exists an $H \in \calM_k$ such that
\begin{equation}\label{eq:(wx)-upper}
(w \cdot x)^2 \leq 2 \big[ b_k^2 + x\trans H x \big].
\end{equation}
Thus,
\begin{equation*}
\sup_{w \in W_k} \frac{1}{\abs{\TC}} \sum_{x \in \TC} (w\cdot x)^2 \leq 2b_k^2 + 2 \sup_{H \in \calM_k}\frac{1}{\abs{\TC}} \sum_{x \in \TC} x\trans H x \leq 5C_2 (b_k^2 + r_k^2),
\end{equation*}
where the last inequality follows from the fact $C_2 \geq 2$.
\end{proof}

\begin{proposition}[Formal statement of Proposition~\ref{prop:outlier-guarantee}]\label{prop:outlier-guarantee-restate}
Consider phase $k$ of Algorithm~\ref{alg:main} for any $1 \leq k \leq k_0$. Suppose that Assumption~\ref{as:x} and \ref{as:y} are satisfied, and that $\eta \leq c_5 \epsilon$. With probability $1- \frac{\delta_k}{8}$ (over the draw of $\bar{T}$), Algorithm~\ref{alg:reweight} will output a function $q: T \rightarrow [0, 1]$ with the following properties:
\begin{enumerate}
\item \label{item:outlier:q} for all $x \in T,\ q(x) \in [0, 1]$;
\item \label{item:outlier:avg-q} $\frac{1}{\abs{T}} \sum_{x \in T} q(x) \geq 1 - \xi_k$;
\item  \label{item:outlier:avg-var} for all $w \in W_k$, $\frac{1}{\abs{T}} \sum_{x \in T} q(x) (w\cdot x)^2 \leq 5 C_2 \del{b_k^2 + r_k^2}$.
\end{enumerate}
Furthermore, such function $q$ can be found in polynomial time.
\end{proposition}
\begin{proof}
Our choice on $n_k$ satisfies the condition $n_k \geq \frac{6}{\xi_k}\ln\frac{48}{\delta_k}$ since $\xi_k$ is lower bounded by a constant (see Section~\ref{subsec:param-setting} for our parameter setting). Thus by Lemma~\ref{lem:pruning}, with probability $1- \frac{\delta_k}{24}$, $\abs{\TC} \geq (1 - \xi_k) n_k$. We henceforth condition on this happening.

On the other side, Lemma~\ref{lem:xHx-emp} and Proposition~\ref{prop:concentrate-spca} together implies that with probability $1-\frac{\delta_k}{24}$, for all $H \in \calM_k$, we have
\begin{equation}\label{eq:tmp:xHx-emp}
\frac{1}{\abs{\TC}} \sum_{x \in \TC} x\trans H x \leq 2 C_2 (b_k^2 + r_k^2)
\end{equation}
provided that
\begin{equation}\label{eq:tmp:TC-size}
\abs{\TC} = \tilde{O}\del[2]{s^2 \log^4 \frac{d}{b_k} \cdot \del[2]{\log d + \log^2 \frac{1}{\delta_k} }}.
\end{equation}
%
Note that \eqref{eq:tmp:TC-size} is satisfied in view of the aforementioned event $\abs{\TC} \geq (1-\xi_k) n_k$ along with the setting of $n_k$ and $\xi_k$. By union bound, the events \eqref{eq:tmp:xHx-emp} and $\abs{\TC} \geq (1-\xi_k)\abs{T}$ hold simultaneously with probability at least $1- \frac{\delta_k}{8}$.

Now we show that these two events together implies the existence of a feasible function $q(x)$ to Algorithm~\ref{alg:reweight}. Consider a particular function $q(x)$ with $q(x) = 0$ for all $x \in \TD$ and $q(x)=1$ for all $x \in \TC$. We immediately have
\begin{equation*}
\frac{1}{\abs{T}} \sum_{x \in T} q(x) = \frac{\abs{\TC}}{\abs{T}} \geq 1 - \xi_k.
\end{equation*}
In addition, for all $H \in \calM_k$,
\begin{equation}\label{eq:tmp:sum-xHx}
\frac{1}{\abs{T}} \sum_{x \in T} q(x) x\trans H x = \frac{1}{\abs{T}} \sum_{x \in \TC} x\trans H x \leq \frac{1}{\abs{\TC}} \sum_{x \in \TC} x\trans H x \leq 2C_2 (b_k^2 + r_k^2),
\end{equation}
where the first inequality follows from the fact $\abs{T} \geq \abs{\TC}$ and the second inequality follows from \eqref{eq:tmp:xHx-emp}. Namely, such function $q(x)$ satisfies all the constraints in Algorithm~\ref{alg:reweight}. Finally, combining \eqref{eq:(wx)-upper} and \eqref{eq:tmp:sum-xHx} gives Part~\ref{item:outlier:avg-var}.

It remains to show that for a given candidate function $q$, a separation oracle for Algorithm~\ref{alg:reweight} can be constructed in polynomial time. First, it is straightforward to check whether the first two constraints $q(x) \in [0, 1]$ and $\sum_{x \in T} q(x) \geq (1-\xi)\abs{T}$ are violated. If not, we just need to further check if there exists an $H \in \calM_k$ such that $\frac{1}{\abs{T}} \sum_{x \in T} q(x) x\trans H x > 2 C_2(b_k^2 + r_k^2)$. To this end, we appeal to solving the following program:
\begin{equation*}
\max_{H \in \calM_k} \frac{1}{\abs{T}} \sum_{x \in T} q(x)  x\trans H x.
\end{equation*}
This is a semidefinite program that can be solved in polynomial time~\citep{boyd2004convex}. If the maximum objective value is greater than $2 C_2(b_k^2 + r_k^2)$, then we conclude that $q$ is not feasible; otherwise we would have found a desired function.
\end{proof}

The analysis of the following proposition closely follows~\citet{awasthi2017power} with a refined treatment. Let $\ell_{\tau_k}(w; p) \defeq \sum_{x \in T} p(x) \ell_{\tau_k}(w; x, y_x)$ where $y_x$ is the unrevealed label of $x$ that the adversary has committed to.

\begin{proposition}[Formal statement of Proposition~\ref{prop:l(TC)=l(p)}]\label{prop:l(TC)=l(p)-restate}
Consider phase $k$ of Algorithm~\ref{alg:main}. Suppose that Assumption~\ref{as:x} and \ref{as:y} are satisfied. Assume that $\eta \leq  c_5\epsilon$. Set $N_k$ and $\xi_k$ as in Section~\ref{subsec:param-setting}. Denote $z_k \defeq \sqrt{b_k^2 + r_k^2} = \sqrt{\bar{c}^2 + 1} \cdot 2^{-k-3}$. With probability $1- \frac{\delta_k}{4}$ over the draw of $\bar{T}$, for all $w \in W_k$
\begin{align*}
\ell_{\tau_k}(w; \tildeTC) &\leq \ell_{\tau_k}(w; p) + 2 \xi_k \(1 + \sqrt{10C_2} \cdot \frac{z_k}{\tau_k}\)  + \sqrt{10 C_2 \xi_k} \cdot \frac{z_k}{\tau_k},\\
\ell_{\tau_k}(w; p) &\leq \ell_{\tau_k}(w; \tildeTC) + 2 \xi_k + \sqrt{20 C_2 \xi_k } \cdot \frac{z_k}{\tau_k}.
\end{align*}
In particular, with our hyper-parameter setting,
\begin{equation*}
\abs{\ell_{\tau_k}(w; \tildeTC) - \ell_{\tau_k}(w; p)} \leq \kappa.
\end{equation*}
\end{proposition}
\begin{proof}
The choice of $n_k$ guarantees that Lemma~\ref{lem:pruning} and Proposition~\ref{prop:outlier-guarantee-restate} hold simultaneously with probability $1- \frac{\delta_k}{4}$. We thus have for all $w \in W_k$
\begin{align}
\frac{1}{\abs{T}} \sum_{x \in T} q(x)(w\cdot x)^2 &\leq 5C_2 z_k^2,\label{eq:q(x)wx}\\
\frac{1}{\abs{\TC}}\sum_{x \in \TC}(w \cdot x)^2 &\leq 5C_2 z_k^2,\label{eq:wx}\\
\frac{\abs{\TD}}{\abs{T}} &\leq \xi_k.\label{eq:|T_D|/|T|}
\end{align}
In the above expression, \eqref{eq:q(x)wx} and \eqref{eq:wx} follow from Part~\ref{item:outlier:avg-var} and Part~\ref{item:outlier:avg-var-TC} of Lemma~\ref{lem:xHx-emp-restate} respectively, \eqref{eq:|T_D|/|T|} follows from Lemma~\ref{lem:pruning}. It follows from Eq.~\eqref{eq:|T_D|/|T|} and $\xi_k \leq 1/2$ that
\begin{equation}\label{eq:tmp:|W|/|T_C|}
\frac{\abs{T}}{\abs{\TC}}  = \frac{\abs{T}}{\abs{T} - \abs{\TD}}  = \frac{1}{1 - \abs{\TD} / \abs{T}} \leq \frac{1}{1- \xi_k} \leq 2.
\end{equation}
In the following, we condition on the event that all these inequalities are satisfied.

\vspace{0.1in}
\noindent{\bfseries Step 1.}
First we upper bound $\ell_{\tau_k}(w; \tildeTC)$ by $\ell_{\tau_k}(w; p)$.
\begin{align}
\abs{\TC} \cdot \ell_{\tau_k}(w; \tildeTC) &=  \sum_{x \in \TC} \ell(w; x, y_x) \notag\\
&= \sum_{x \in T} \sbr{ q(x) \ell(w; x, y_x) + \big(\ind{x \in \TC} - q(x)\big) \ell(w; x, y_x) } \notag\\
&\stackrel{\zeta_1}{\leq}   \sum_{x \in T} q(x) \ell(w; x, y_x) + \sum_{x \in \TC} (1 - q(x)) \ell(w; x, y_x) \notag\\
&\stackrel{\zeta_2}{\leq}  \sum_{x \in T} q(x) \ell(w; x, y_x) + \sum_{x \in \TC} (1 - q(x)) \del{1 + \frac{\abs{w\cdot x}}{\tau_k}} \notag\\
&\stackrel{\zeta_3}{\leq} \sum_{x \in T} q(x) \ell(w; x, y_x) + \xi_k \abs{T} + \frac{1}{\tau_k} \sum_{x \in \TC} (1-q(x)) \abs{w\cdot x} \notag\\
&\stackrel{\zeta_4}{\leq} \sum_{x \in T} q(x) \ell(w; x, y_x) + \xi_k \abs{T} + \frac{1}{\tau_k} \sqrt{\sum_{x \in \TC} (1-q(x))^2} \cdot \sqrt{\sum_{x \in \TC} (w\cdot x)^2} \notag\\
&\stackrel{\zeta_5}{\leq} \sum_{x \in T} q(x) \ell(w; x, y_x) + \xi_k \abs{T} + \frac{1}{\tau_k} \sqrt{\xi_k \abs{T}} \cdot  \sqrt{5C_2 \abs{\TC}} \cdot {z_k}, \label{eq:tmp:l(T_C)-1}
\end{align}
where $\zeta_1$ follows from the simple fact that 
\begin{align*}
\sum_{x \in T} \del[1]{\ind{x \in \TC} - q(x)} \ell(w; x, y_x) &= \sum_{x \in \TC} (1-q(x)) \ell(w; x, y_x) + \sum_{x \in \TD} (-q(x)) \ell(w; x, y_x) \\
&\leq \sum_{x \in \TC} (1-q(x)) \ell(w; x, y_x),
\end{align*}
$\zeta_2$ explores the fact that the hinge loss is always upper bounded by $1 + \frac{\abs{w\cdot x}}{\tau_k}$ and that $1-q(x) \geq 0$, $\zeta_3$ follows from Part~\ref{item:outlier:avg-q} of Proposition~\ref{prop:outlier-guarantee-restate}, $\zeta_4$ applies Cauchy-Schwarz inequality, and $\zeta_5$ uses Eq.~\eqref{eq:wx}.

In view of Eq.~\eqref{eq:tmp:|W|/|T_C|}, we have $\frac{\abs{T}}{\abs{\TC}} \leq 2$. Continuing Eq.~\eqref{eq:tmp:l(T_C)-1}, we obtain
\begin{align}
\ell_{\tau_k}(w; \tildeTC) &\leq \frac{1}{\abs{\TC}}\sum_{x \in T} q(x) \ell(w; x, y_x) + 2 \xi_k + \sqrt{10 C_2 \xi_k} \cdot \frac{z_k}{\tau_k} \notag\\
&= \frac{\sum_{x \in T} q(x)}{\abs{\TC}} \sum_{x \in T} p(x) \ell(w; x, y_x) + 2 \xi_k + \sqrt{10 C_2 \xi_k} \cdot \frac{z_k}{\tau_k} \notag\\
&= \ell_{\tau_k}(w; p) + \del{\frac{\sum_{x \in T} q(x)}{\abs{\TC}} - 1} \sum_{x \in T} p(x) \ell(w; x, y_x) + 2 \xi_k + \sqrt{10 C_2 \xi_k} \cdot \frac{z_k}{\tau_k} \notag\\
&\leq \ell_{\tau_k}(w; p) + \del{\frac{\abs{T}}{\abs{\TC}} - 1} \sum_{x \in T} p(x) \ell(w; x, y_x) + 2 \xi_k + \sqrt{10 C_2 \xi_k} \cdot \frac{z_k}{\tau_k} \notag\\
&\leq \ell_{\tau_k}(w; p) + 2\xi_k \sum_{x \in T} p(x) \ell(w; x, y_x) + 2 \xi_k + \sqrt{10 C_2 \xi_k} \cdot \frac{z_k}{\tau_k},\label{eq:tmp:l(T_C)-2}
\end{align}
where in the last inequality we use $\abs{T}/\abs{\TC}-1 = \frac{\abs{\TD}/ \abs{T}}{1 - \abs{\TD}/\abs{T}} \leq 2 \abs{\TD}/\abs{T}$. On the other hand, we have the following result which will be proved later on.

\begin{claim}\label{claim:err-aux}
$\sum_{x \in T} p(x) \ell(w; x, y_x) \leq 1 + \sqrt{10 C_2} \cdot \frac{z_k}{\tau_k}.$
\end{claim}

Therefore, continuing Eq.~\eqref{eq:tmp:l(T_C)-2} we have
\begin{equation*}
\ell_{\tau_k}(w; \tildeTC) \leq \ell_{\tau_k}(w; p) + 2 \xi_k \(1 + \sqrt{10 C_2} \cdot \frac{z_k}{\tau_k}\)  + \sqrt{10 C_2 \xi_k} \cdot \frac{z_k}{\tau_k}.
\end{equation*}
which proves the first inequality of the proposition.

\vspace{0.1in}
\noindent{\bfseries Step 2.}
We move on to prove the second inequality of the theorem, i.e. using $\ell_{\tau_k}(w; \tildeTC)$ to upper bound $\ell_{\tau_k}(w; p)$. Let us denote by $\pD = \sum_{x \in \TD} p(x)$ the probability mass on dirty instances. Then
\begin{equation}\label{eq:tmp:p_D}
\pD = \frac{\sum_{x \in \TD} q(x)}{\sum_{x \in T} q(x)} \leq \frac{\abs{\TD}}{(1-\xi_k)\abs{T}}\leq \frac{\xi_k}{1-\xi_k} \leq 2 \xi_k,
\end{equation}
where the first inequality follows from $q(x) \leq 1$ and Part~\ref{item:outlier:avg-q} of Proposition~\ref{prop:outlier-guarantee-restate}, the second inequality follows from \eqref{eq:|T_D|/|T|}, and the last inequality is by our choice $\xi_k \leq 1/2$.

Note that by Part~\ref{item:outlier:avg-q} of Proposition~\ref{prop:outlier-guarantee-restate} and the choice $\xi_k \leq 1/2$, we have $\sum_{x \in T} q(x) \geq (1-\xi_k) \abs{T} \geq \abs{T}/2$. Hence
\begin{equation}\label{eq:tmp:sum p(x)(wx)^2}
\sum_{x \in T} p(x) (w\cdot x)^2 = \frac{1}{\sum_{x \in T} q(x)} \sum_{x \in T} q(x) (w\cdot x)^2 \leq \frac{2}{\abs{T}}  \sum_{x \in T} q(x) (w\cdot x)^2 \leq 10 C_2 z_k^2
\end{equation}
where the last inequality holds because of \eqref{eq:q(x)wx}. Thus,
\begin{align*}
\sum_{x \in \TD} p(x) \ell(w; x, y_x) &\leq \sum_{x \in \TD} p(x) \del{1 + \frac{\abs{w\cdot x}}{\tau_k}}\\
&= \pD + \frac{1}{\tau_k}\sum_{x \in \TD} p(x) \abs{w\cdot x}\\
&= \pD +\frac{1}{\tau_k} \sum_{x \in T} \del {\ind{x \in \TD} \sqrt{p(x)} } \cdot \del{ \sqrt{p(x)}\abs{w\cdot x} }\\
&\leq \pD +  \frac{1}{\tau_k}  \sqrt{\sum_{x \in T} \ind{x \in \TD} {p(x)}} \cdot \sqrt{ \sum_{x \in T} p(x) (w\cdot x)^2 }\\
&\stackrel{\eqref{eq:tmp:sum p(x)(wx)^2}}{\leq} \pD + \sqrt{\pD} \cdot \sqrt{10 C_2} \cdot  \frac{z_k}{\tau_k}.
\end{align*}
With the result on hand, we bound $\ell_{\tau_k}(w; p)$ as follows:
\begin{align*}
\ell_{\tau_k}(w; p) &= \sum_{x \in \TC} p(x) \ell(w; x, y_x) + \sum_{x \in \TD} p(x) \ell(w; x, y_x)\\
&\leq \sum_{x \in \TC} \ell(w; x, y_x) +  \sum_{x \in \TD} p(x) \ell(w; x, y_x)\\
&= \ell_{\tau_k}(w; \tildeTC) +  \sum_{x \in \TD} p(x) \ell(w; x, y_x)\\
&\leq \ell_{\tau_k}(w; \tildeTC) + \pD + \sqrt{\pD} \cdot \sqrt{10 C_2} \cdot  \frac{z_k}{\tau_k}\\
&\stackrel{\eqref{eq:tmp:p_D}}{\leq} \ell_{\tau_k}(w; \tildeTC) + 2 \xi_k + \sqrt{20 C_2 \xi_k } \cdot \frac{z_k}{\tau_k},
\end{align*}
which proves the second inequality of the proposition.

\vspace{0.1in}
\noindent{\bfseries Putting together.}
We would like to show $\abs{\ell_{\tau_k}(w; p) - \ell_{\tau_k}(w; \tildeTC)} \leq \kappa$. Indeed, this is guaranteed by our setting of $\xi_k$ in Section~\ref{subsec:param-setting} which ensures that $\xi_k$ simultaneously fulfills the following three constraints:
\begin{align*}
2 \xi_k \(1 + \sqrt{10 C_2} \cdot \frac{z_k}{\tau_k}\)  + \sqrt{10 C_2 \xi_k} \cdot \frac{z_k}{\tau_k} \leq \kappa,\\
2 \xi_k + \sqrt{20 C_2 \xi_k } \cdot \frac{z_k}{\tau_k} \leq \kappa, \quad \text{and}\quad \xi_k \leq \frac{1}{2}.
\end{align*}
This completes the proof.
\end{proof}

\begin{proof}[Proof of Claim~\ref{claim:err-aux}]
Since $\ell(w; x, y_x) \leq 1 + \frac{\abs{w\cdot x}}{\tau_k}$, it follows that
\begin{align*}
\sum_{x \in T} p(x) \ell(w; x, y_x) &\leq \sum_{x \in T} p(x) \del{ 1 + \frac{\abs{w\cdot x}}{\tau_k}}\\
&= 1 + \frac{1}{\tau_k} \sum_{x \in T} p(x) \abs{w\cdot x}\\
&\leq 1 + \frac{1}{\tau_k} \sqrt{ \sum_{x \in T} p(x) (w\cdot x)^2}\\
&\stackrel{\eqref{eq:tmp:sum p(x)(wx)^2}}{\leq} 1 + \sqrt{10 C_2} \cdot \frac{z_k}{\tau_k},
\end{align*}
which completes the proof of Claim~\ref{claim:err-aux}.
\end{proof}

The following result is a simple application of Proposition~\ref{prop:concentrate-hinge}. It shows that the loss evaluated on clean instances concentrates around the expected loss.

\begin{proposition}[Restatement of Proposition~\ref{prop:l(TC)=exp}]\label{prop:l(TC)=exp-restate}
Consider phase $k$ of Algorithm~\ref{alg:main}. Suppose that Assumption~\ref{as:x} and \ref{as:y} are satisfied, and assume $\eta \leq c_5 \epsilon$. Then with probability $1 - \frac{\delta_k}{4}$ over the draw of $\bar{T}$, for all $w \in W_k$ we have
\begin{equation*}
\abs{L_{\tau_k}(w) - \ell_{\tau_k}(w; \tildeTC) } \leq \kappa.
\end{equation*}
where $L_{\tau_k}(w) \defeq \EXP_{x \sim D_{w_{k-1}, b_{k}}}\sbr{\ell_{\tau_k}(w; x, \sign{w^* \cdot x})}$.
\end{proposition}
\begin{proof}
The choice of $n_k$, i.e. the size of $\abs{\bar{T}}$, ensures that with probability $1-\frac{\delta_k}{8}$, $\abs{\TC}$ is at least $\zeta  \log \zeta$ where $\zeta = K \cdot s \log^2 \frac{d}{b_k} \cdot \log \frac{d}{\delta_k}$ for some constant $K>0$ in view of Lemma~\ref{lem:pruning}. This observation in allusion to Proposition~\ref{prop:concentrate-hinge} and union bound, immediately gives the desired result.
\end{proof}

\subsection{Analysis of random sampling}\label{sec:app:random-sampling}

\begin{proposition}[Restatement of Proposition~\ref{prop:l(p)=l(S)}]\label{prop:l(p)=l(S)-restate}
Consider phase $k$ Algorithm~\ref{alg:main}. Suppose that Assumption~\ref{as:x} and \ref{as:y} are satisfied, and assume $\eta \leq c_5 \epsilon$. Set $n_k$ and $m_k$ as in Section~\ref{subsec:param-setting}. Then with probability $1 - \frac{\delta_k}{4}$ over the draw of $S_k$, for all $w \in W_k$ we have
\begin{equation*}
\abs{\ell_{\tau_k}(w; p) - \ell_{\tau_k}(w; S_k)} \leq \kappa.
\end{equation*}
\end{proposition}
\begin{proof}

Since we applied pruning to remove all instances with large $\ell_{\infty}$-norm, this proposition can be proved by a standard concentration argument for uniform convergence of linear classes under distributions with $\ell_\infty$ bounded support. We include the proof for completeness.

Note that the randomness is taken over the i.i.d. draw of $m_k$ samples from $T$ according to the distribution $p$ over $T$. Thus, for any $(x, y) \in S_k$, $\EXP[\ell_{\tau_k}(w; x, y)] = \ell_{\tau_k}(w; p)$. Moreover, let $R_k = \max_{x \in T} \infnorm{x}$. Any instance $x$ drawn from $T$ satisfies $\infnorm{x} \leq R_k$ with probability $1$. It is also easy to verify that
\begin{equation*}
\ell_{\tau_k}(w; x, y) \leq 1 + \frac{\abs{w\cdot x}}{\tau_k} \leq 1 + \frac{(w - w_{k-1})\cdot x}{\tau_k} + \frac{\abs{w_{k-1} \cdot x}}{\tau_k} \leq 1 + \frac{\rho_k R_k}{\tau_k} + \frac{b_k}{\tau_k}.
\end{equation*}
By Theorem~8 of \citet{bartlett2002rademacher} along with standard symmetrization arguments, we have that with probability at least $1 - \frac{\delta_k}{4}$,
\begin{equation}\label{eq:tmp:rad-1}
\abs{\ell_{\tau_k}(w; p) - \ell_{\tau_k}(w; S_k)} \leq \del{1 + \frac{\rho_k R_k}{\tau_k} + \frac{b_k}{\tau_k} } \sqrt{\frac{\ln(4/\delta_k)}{2 m_k}} + \calR(\calF; S_k)
\end{equation}
where $\calR(\calF; S_k)$ denotes the Rademacher complexity of function class $\calF$ on the labeled set $S_k$, and $\calF \defeq \cbr{ \ell_{\tau_k}(w; x, y): w \in W_k }$. In order to calculate $\calR(\calF; S_k)$, we observe that each function $\ell_{\tau_k}(w; x, y)$ is a composition of $\phi(a) = \max\cbr{0, 1 - \frac{1}{\tau_k}y a}$ and function class $\calG \defeq \{ x \mapsto w \cdot x: w \in W_k \}$. Since $\phi(a)$ is $\frac{1}{\tau_k}$-Lipschitz, by contraction property of Rademacher complexity, we have
\begin{equation}\label{eq:tmp:rad-2}
\calR(\calF; S_k) \leq \frac{1}{\tau_k} \calR(\calG; S_k).
\end{equation}
Let $\sigma = \{\sigma_1, \dots, \sigma_{m_k}\}$ where the $\sigma_i$'s are i.i.d. draw from the Rademacher distribution, and let $V_k = B_2(0, r_k) \cap B_1(0, \rho_k)$. We compute $\calR(\calG; S_k)$ as follows:
\begin{align*}
\calR(\calG; S_k) &= \frac{1}{m_k} \EXP_{\sigma}\sbr[3]{ \sup_{w \in W_k} w \cdot \del[3]{\sum_{i=1}^{m_k} \sigma_i x_i } }\\
&= \frac{1}{m_k} \EXP_{\sigma}\sbr[3]{ w_{k-1} \cdot \del[3]{\sum_{i=1}^{m_k} \sigma_i x_i } } + \frac{1}{m_k} \EXP_{\sigma}\sbr[3]{ \sup_{w \in W_k} (w - w_{k-1}) \cdot \del[3]{\sum_{i=1}^{m_k} \sigma_i x_i } }\\
&= \frac{1}{m_k} \EXP_{\sigma}\sbr[3]{ \sup_{v \in V_k} v \cdot \del[3]{\sum_{i=1}^{m_k} \sigma_i x_i } } \\
&\leq \rho_k R_k \sqrt{ \frac{2\log(2d)}{m_k} },
\end{align*}
where the first equality is by the definition of Rademacher complexity, the second equality simply decompose $w$ as a sum of $w_{k-1}$ and $w - w_{k-1}$, the third equality is by the fact that every $\sigma_i$ has zero mean, and the inequality applies Lemma~\ref{lem:L1-rad}. We combine the above result with \eqref{eq:tmp:rad-1} and \eqref{eq:tmp:rad-2}, and obtain that with probability $1- \frac{\delta_k}{4}$,
\begin{equation}\label{eq:tmp:rad-3}
\abs{\ell_{\tau_k}(w; p) - \ell_{\tau_k}(w; S_k)} \leq \del{1 + \frac{\rho_k R_k}{\tau_k} + \frac{b_k}{\tau_k}} \sqrt{ \frac{ \ln(4/\delta_k)}{m_k} } + \frac{\rho_k R_k}{\tau_k} \sqrt{\frac{2\log(2d)}{m_k}}.
\end{equation}
Recall that we remove all instances with large $\ell_{\infty}$-norm in the pruning step of Algorithm~\ref{alg:main}. In particular, we have
\begin{equation*}
R_k \leq c_9 \log\frac{48 n_k d}{b_k \delta_k }.
\end{equation*}
Plugging this upper bound into \eqref{eq:tmp:rad-3} and using our hyper-parameter setting gives
\begin{equation*}
\abs{\ell_{\tau_k}(w; p) - \ell_{\tau_k}(w; S_k)} \leq K_1 \cdot \sqrt{s} \log \frac{n_k d}{b_k \delta_k } \del{ \sqrt{ \frac{\log(1/\delta_k)}{m_k} }  + \sqrt{\frac{\log d}{m_k}}}
\end{equation*}
for some constant $K_1 > 0$. Hence,
\begin{equation*}
m_k = {O}\del{ s \log^2\frac{n_k d}{b_k \delta_k} \cdot \log\frac{d}{\delta_k} } = \tilde{O}\del{ s \log^2 \frac{d}{b_k \delta_k} \cdot \log\frac{d}{\delta_k} }
\end{equation*}
suffices to ensure $\abs{\ell_{\tau_k}(w; p) - \ell_{\tau_k}(w; S_k)} \leq \kappa$ with probability $1- \frac{\delta_k}{4}$.
\end{proof}

\subsection{Analysis of Per-Phase Progress}\label{sec:app:progress}

Let $L_{\tau_k}(w) = \EXP_{x \sim D_{w_{k-1}, b_{k}}} \sbr{ \ell_{\tau_k}(w; x, \sign{w^* \cdot x}) }$.

\begin{lemma}[Lemma~3.7 of \citet{awasthi2017power}]\label{lem:L(w*)}
Suppose Assumption~\ref{as:x} is satisfied. Then
\begin{equation*}
L_{\tau_k}(w^*)  \leq \frac{\tau_k}{c_0 \min\{b_k, 1/9\}}.
\end{equation*}
In particular, by our choice of $\tau_k$
\begin{equation*}
L_{\tau_k}(w^*) \leq \kappa.
\end{equation*}
\end{lemma}

\begin{lemma}\label{lem:err_k(v_k)}
For any $1 \leq k \leq k_0$, if $w^* \in W_k$, then with probability $1- \delta_k$, $\err_{D_{w_{k-1}, b_{k}}}(v_k) \leq 8\kappa$.
\end{lemma}
\begin{proof}
Observe that with the setting of $N_k$, we have with probability $1-\delta_k$ over all the randomness in phase $k$, Lemma~\ref{lem:N-restate}, Proposition~\ref{prop:l(TC)=l(p)-restate}, Proposition~\ref{prop:l(TC)=exp-restate} and Proposition~\ref{prop:l(p)=l(S)-restate} hold simultaneously. Now we condition on the event that all of these properties are satisfied, which implies for all $w \in W_k$,
\begin{equation}\label{eq:tmp:exp=l(S)}
\abs{L_{\tau_k}(w) - \ell_{\tau_k}(w; S_k) } \leq 3\kappa.
\end{equation}
We have
\begin{align*}
\err_{D_{w_{k-1}, b_{k}}}(v_k) \leq L_{\tau_k}(v_k) \stackrel{\zeta_1}{\leq} \ell_{\tau_k}(v_k; S_k) + 3\kappa \stackrel{\zeta_2}{\leq} \min_{w \in W_k} \ell_{\tau_k}(w; S_k) + 4\kappa &\stackrel{\zeta_3}{\leq} \ell_{\tau_k}(w^*; S_k) + 4\kappa\\
&\leq L_{\tau_k}(w^*) + 7\kappa.
\end{align*}
In the above, the first inequality follows from the fact that hinge loss upper bounds the 0/1 loss, $\zeta_1$ and the last inequality applies \eqref{eq:tmp:exprad}, $\zeta_2$ is by the definition of $v_k$ (see Algorithm~\ref{alg:main}), and $\zeta_3$ is by our assumption that $w^*$ is feasible. The proof is complete in view of Lemma~\ref{lem:L(w*)}.
\end{proof}

\begin{lemma}\label{lem:feasible-to-angle}
For any $1 \leq k \leq k_0$, if $w^* \in W_k$, then with probability $1-\delta_k$, $\theta(v_k, w^*) \leq 2^{-k-8} \pi$.
\end{lemma}
\begin{proof}

For $k = 1$, by Lemma~\ref{lem:err_k(v_k)} and that we actually sample from $D$, we have
\begin{equation*}
\Pr_{x \sim D}\(\sign{v_1 \cdot x} \neq \sign{w^* \cdot x}\) \leq 8\kappa.
\end{equation*}
Hence Part~\ref{item:ilc:err=theta} of Lemma~\ref{lem:logconcave} indicates that
\begin{equation}\label{eq:tmp:theta_1}
\theta(v_1, w^*) \leq 8c_2 \kappa = 16 c_2 \kappa \cdot 2^{-1}.
\end{equation}

Now we consider $2 \leq k \leq  k_0$. Denote $X_k = \{x: \abs{w_{k-1}\cdot x} \leq b_k\}$, and $\bar{X}_k = \{x: \abs{w_{k-1}\cdot x} > b_k\}$. We will show that the error of $v_k$ on both $X_k$ and $\bar{X}_k$ is small, hence $v_k$ is a good approximation to $w^*$.

First, we consider the error on $X_k$, which is given by
\begin{align}
 &\ \Pr_{x \sim D}\( \sign{v_k \cdot x} \neq \sign{w^* \cdot x}, x \in X_k \) \notag\\
=&\ \Pr_{x \sim D}\( \sign{v_k \cdot x} \neq \sign{w^* \cdot x} \mid x \in X_k \) \cdot \Pr_{x \sim D}(x \in X_k) \notag\\
=&\ \err_{D_{w_{k-1}, b_{k}}}(v_k) \cdot \Pr_{x \sim D}(x \in X_k) \notag\\
\leq&\ 8\kappa \cdot 2b_k\notag\\
=&\ 16 \kappa b_k,\label{eq:tmp:err in band}
\end{align}
where the inequality is due to Lemma~\ref{lem:err_k(v_k)} and Lemma~\ref{lem:P(x in band)}. Note that the inequality holds with probability $1 - \delta_k$.

Next we derive the error on $\bar{X}_k$. Note that Lemma~10 of \citet{zhang2018efficient} states for any unit vector $u$, and any general vector $v$, $\theta(v, u) \leq \pi \twonorm{v - u}$. Hence,
\begin{align*}
\theta(v_k, w^*)  \leq \pi \twonorm{v_k - w^*} \leq \pi (\twonorm{v_k - w_{k-1}} + \twonorm{w^* - w_{k-1}}) \leq 2 \pi r_k.
\end{align*}
Recall that we set $r_k = 2^{-k-3} < 1/4$ in our algorithm and choose $b_k = \bar{c} \cdot r_k$ where $\bar{c} \geq 8\pi / c_4$, which allows us to apply Lemma~\ref{lem:err outside band} and obtain
\begin{align*}
\Pr_{x \sim D}\( \sign{v_k \cdot x} \neq \sign{w^* \cdot x}, x \notin X_k \) &\leq c_3 \cdot 2\pi r_k \cdot \exp\del{- \frac{c_4 \bar{c} \cdot r_k}{2 \cdot 2\pi r_k}} \\
&= 2^{-k} \cdot  \frac{c_3 \pi}{4} \exp\del{- \frac{c_4 \bar{c}}{4 \pi}}.
\end{align*}
This in allusion to \eqref{eq:tmp:err in band} gives
\begin{equation*}
\err_D(v_k) \leq 16\kappa \cdot \bar{c} \cdot r_k + 2^{-k} \cdot  \frac{c_3 \pi}{4} \exp\del{- \frac{c_4 \bar{c}}{4 \pi}} = \del{ 2\kappa \bar{c} + \frac{c_3 \pi}{4} \exp\del{- \frac{c_4 \bar{c}}{4 \pi}} } \cdot 2^{-k}.
\end{equation*}
Recall that we set $\kappa = \exp(-\bar{c})$ and denote by $f(\bar{c})$ the coefficient of $2^{-k}$ in the above expression. By Part~\ref{item:ilc:err=theta} of Lemma~\ref{lem:logconcave}
\begin{equation}\label{eq:tmp:theta_k}
\theta(v_k, w^*) \leq c_2 \err_D(v_k) \leq c_2 f(\bar{c}) \cdot 2^{-k}.
\end{equation}

Now let $g(\bar{c}) = c_2 f(\bar{c}) + 16c_2 \exp(-\bar{c})$. By our choice of $\bar{c}$, $g(\bar{c}) \leq 2^{-8}\pi$. This ensures that for both \eqref{eq:tmp:theta_1} and \eqref{eq:tmp:theta_k}, $\theta(v_k, w^*) \leq 2^{-k-8}\pi$ for any $k \geq 1$.
\end{proof}

\begin{lemma}\label{lem:angle-to-feasible}
For any $1 \leq k \leq k_0$, if $\theta(v_k, w^*) \leq 2^{-k-8}\pi$, then $w^* \in W_{k+1}$.
\end{lemma}
\begin{proof}
We first show that $\twonorm{w_k - w^*} \leq r_{k+1}$. Let $\hat{v}_k = v_k / \twonorm{v_k}$. By algebra $\twonorm{\hat{v}_k - w^*} = 2 \sin\frac{\theta(v_k, w^*)}{2} \leq \theta(v_k, w^*) \leq 2^{-k-8}\pi \leq 2^{-k-6}$. Now we have
\begin{align*}
\twonorm{w_{k} - w^*} &= \twonorm{\calH_s(v_k) / \twonorm{\calH_s(v_k) } - w^*}\\
&= \twonorm{\calH_s(\hat{v}_k) / \twonorm{\calH_s(\hat{v}_k)} - w^*}\\
&\leq 2 \twonorm{\calH_s(\hat{v}_k) - w^*}\\
&\leq 4 \twonorm{\hat{v}_k - w^*}\\
&\leq 2^{-k-4}\\
&= r_{k+1}.
\end{align*}
By the sparsity of $w_k$ and $w^*$, and our choice $\rho_{k+1} = \sqrt{2s}r_{k+1}$, we always have
\begin{equation*}
\onenorm{w_k - w^*} \leq \sqrt{2s} \twonorm{w_k - w^*} \leq \sqrt{2s} r_{k+1} = \rho_{k+1}.
\end{equation*}
The proof is complete.
\end{proof}

\subsection{Proof of Theorem~\ref{thm:main}}\label{sec:app:main-proof}
\begin{proof}
We will prove the theorem with the following claim.

\begin{claim}\label{claim:main-aux}
For any $1 \leq k \leq k_0$, with probability at least $1 - \sum_{i=1}^{k} \delta_i$, $w^*$ is in $W_{k+1}$.
\end{claim}

Based on the claim, we immediately have that with probability at least $1 - \sum_{k=1}^{k_0} \delta_k \geq 1 - \delta$, $w^*$ is in $W_{k_0 + 1}$. By our construction of $W_{k_0 + 1}$, we have
\begin{equation*}
\twonorm{w^* - w_{k_0}} \leq 2^{-k_0-4}.
\end{equation*}
This, together with Part~\ref{item:ilc:err=theta} of Lemma~\ref{lem:logconcave} and the fact that $\theta(w^*, w_{k_0}) \leq \pi \twonorm{w^* - w_{k_0}}$ (see Lemma~10 of \citet{zhang2018efficient}), implies
\begin{equation*}
\err_D(w_{k_0}) \leq \frac{\pi}{c_1} \cdot 2^{-k_0-4} = \epsilon.
\end{equation*}

Finally, we derive the sample complexity and label complexity. Recall that $n_k$ was involved in Proposition~\ref{prop:outlier-guarantee-restate}, i.e. the quantity $\abs{T}$, where we required
\begin{equation*}
n_k = \tilde{O}\del[3]{s^2 \log^4\frac{d}{b_k} \cdot \del[2]{\log d + \log^2\frac{1}{\delta_k}} + \log\frac{1}{\delta_k} } = \tilde{O}\del[3]{s^2 \log^4\frac{d}{b_k} \cdot \del[2]{\log d + \log^2\frac{1}{\delta_k} }}.
\end{equation*}
It is also involved in Proposition~\ref{prop:l(p)=l(S)-restate}, where we need
\begin{equation*}
m_k = {O}\del{ s \log^2\frac{n_k d}{b_k \delta_k} \cdot \log\frac{d}{\delta_k} }
\end{equation*}
and $n_k \geq m_k$ since $S_k$ is a labeled subset of $T$. As $m_k$ has a cubic dependence on $\log\frac{1}{\delta_k}$, our final choice of $n_k$ is given by
\begin{equation}
n_k = \tilde{O}\del[3]{s^2 \log^4\frac{d}{b_k} \cdot \del[2]{\log d + \log^3\frac{1}{\delta_k}}}.
\end{equation}
This in turn gives
\begin{equation}
m_k = \tilde{O}\del{ s \log^2\frac{d}{b_k \delta_k} \cdot \log\frac{d}{\delta_k} }.
\end{equation}
Therefore, by Lemma~\ref{lem:N-restate} we obtain an upper bound of the sample size $N_k$ at phase $k$ as follows:
\begin{equation*}
N_k = \tilde{O}\del[3]{\frac{s^2}{b_k} \log^4\frac{d}{b_k} \cdot \del[2]{\log d + \log^3\frac{1}{\delta_k}}} \leq \tilde{O}\del[3]{ \frac{s^2}{\epsilon} \log^4 d \del[2]{\log d  + \log^3 \frac{1}{\delta}}},
\end{equation*}
where the last inequality follows from $b_k = \Omega(\epsilon)$ for all $k \leq k_0$ and our choice of $\delta_k$. Consequently, the total sample complexity
\begin{equation*}
N = \sum_{k=1}^{k_0} N_k \leq k_0 \cdot\tilde{O}\del[3]{ \frac{s^2}{\epsilon} \log^4 d \del[2]{\log d  + \log^3 \frac{1}{\delta}}} = \tilde{O}\del[3]{ \frac{s^2}{\epsilon} \log^4 d \del[2]{\log d  + \log^3 \frac{1}{\delta}}}.
\end{equation*}
Likewise, we can show that the total label complexity
\begin{equation*}
m = \sum_{k=1}^{k_0} m_k \leq k_0 \cdot \tilde{O}\del[2]{ s \log^2\frac{d}{\epsilon \delta} \cdot \log\frac{d}{\delta} } = \tilde{O}\del[2]{ s \log^2\frac{d}{\epsilon\delta} \cdot \log\frac{d}{\delta} \cdot\log\frac{1}{\epsilon}}.
\end{equation*}

It remains to prove Claim~\ref{claim:main-aux} by induction. First, for $k=1$, $W_1 = B_2(0, 1) \cap B_1(0, \sqrt{s})$. Therefore, $w^* \in W_1$ with probability $1$. Now suppose that Claim~\ref{claim:main-aux} holds for some $k \geq 2$, that is, there is an event $E_{k-1}$ that happens with probability $1 - \sum_{i}^{k-1}\delta_i$, and on this event $w^* \in W_k$. By Lemma~\ref{lem:feasible-to-angle} we know that there is an event $F_k$ that happens with probability $1-\delta_k$, on which $\theta(v_k, w^*) \leq 2^{-k-8} \pi$. This further implies that $w^* \in W_{k+1}$ in view of Lemma~\ref{lem:angle-to-feasible}. Therefore, consider the event $E_{k-1} \cap F_k$, on which $w^* \in W_{k+1}$ with probability $\Pr(E_{k-1}) \cdot \Pr(F_k \mid E_{k-1}) = (1 - \sum_{i}^{k-1}\delta_i) (1 - \delta_k) \geq 1 - \sum_{i=1}^{k} \delta_i$.
\end{proof}

\section{Miscellaneous Lemmas}

\begin{lemma}[Chernoff bound]\label{lem:chernoff}
Let $Z_1, Z_2, \dots, Z_n$ be $n$ independent random variables that take value in $\{0, 1\}$. Let $Z = \sum_{i=1}^{n} Z_i$. For each $Z_i$, suppose that $\Pr(Z_i =1) \leq \eta$.  Then for any $\alpha \in [0, 1]$
\begin{equation*}
\Pr\( Z \geq  (1+\alpha) \eta n\) \leq e^{-\frac{\alpha^2 \eta n}{3} }.
\end{equation*}
When $\Pr(Z_i =1) \geq \eta$, for any $\alpha \in [0, 1]$
\begin{equation*}
\Pr\( Z \leq  (1-\alpha) \eta n\) \leq e^{-\frac{\alpha^2 \eta n}{2} }.
\end{equation*}
\end{lemma}

\begin{lemma}[Theorem~1 of \citet{kakade2008complexity}]
\label{lem:L1-rad}
Let $\sigma = (\sigma_1, \dots, \sigma_n)$ where $\sigma_i$'s are independent draws from the Rademacher distribution and let $x_1, \dots, x_n $ be given instances in $\Rd$. Then
\begin{equation*}
\EXP_{\sigma} \sbr[3]{ \sup_{w \in B_1(0, \rho)} \sum_{i=1}^{n} \sigma_i w \cdot x_i } \leq \rho \sqrt{2n \log(2d)}\max_{1 \leq i \leq n} \infnorm{x_i}.
\end{equation*}
\end{lemma}

\end{document}